\documentclass[twoside]{article}
\usepackage[preprint]{aistats2027}

\usepackage[numbers]{natbib}

\usepackage[utf8]{inputenc}
\usepackage[T1]{fontenc}
\usepackage{xr-hyper} \usepackage{hyperref}
\usepackage{url}
\usepackage{booktabs}
\usepackage{multirow}
\usepackage{amsfonts}
\usepackage{nicefrac}
\usepackage{microtype}
\usepackage{xcolor}
\usepackage{amsmath,amssymb}
\usepackage{amsthm}
\usepackage{bbm}
\usepackage{graphicx}
\usepackage{algorithm}
\usepackage{algorithmicx}
\usepackage{algpseudocode}
\usepackage{mathtools}
\usepackage{enumitem}
\usepackage[most]{tcolorbox}

\newtheorem{theorem}{Theorem}

\newtheorem{corollary}{Corollary}
\newtheorem{definition}{Definition}

\newtheorem{assumption}{Assumption}

\newcommand{\Fil}{\mathcal{F}}
\newcommand{\ARL}{\mathrm{ARL}_0}
\newcommand{\EDD}{\mathrm{EDD}}

\newcommand{\llmconcept}{\texttt{LLM-concept}}
\newcommand{\embeddingarm}{\texttt{Embedding}}
\newcommand{\saearm}{\texttt{SAE}}
\newcommand{\ensemblearm}{\texttt{Ensemble}}

\newcommand{\informedshift}{\texttt{Informed-shift}}
\newcommand{\unrelatedshift}{\texttt{Unrelated-shift}}
\newcommand{\centroidarm}{\texttt{Centroid}}
\newcommand{\oraclearm}{\texttt{Oracle}}
\newcommand{\E}{\mathbb{E}_{H_0}}
\newcommand{\Ej}{E^{(j)}}

\newcommand{\rj}{r^{(j)}}
\newcommand{\SR}{\text{SR}}
\newcommand{\CU}{\text{CU}}

\newcommand{\Pm}{\hat{P}_m}

\begin{document}

\runningtitle{Learning to Detect Shifts in Unstructured Data Streams via Monitoring-by-Betting}
\runningauthor{Zhou, Ouyang, Xia, Subbaswamy, Prinster, Saria, Garcia-Agundez, Pollet, Yazdany, Vossler, Feng}
\twocolumn[
\aistatstitle{
Learning to Detect Shifts in Unstructured Data Streams\\via Monitoring-by-Betting
}

\vspace{-1cm}
\aistatsauthor{
  \textbf{Siyu Zhou}\textsuperscript{1}, \textbf{Jialin Ouyang}\textsuperscript{1}, \textbf{Fan Xia}\textsuperscript{1}, \textbf{Adarsh Subbaswamy}\textsuperscript{2}, \textbf{Drew Prinster}\textsuperscript{3}, \textbf{Suchi Saria}\textsuperscript{3},  \\
  \textbf{Augusto Garcia-Agundez}\textsuperscript{1}, \textbf{Sarah Pollet}\textsuperscript{1}, \textbf{Jinoos Yazdany}\textsuperscript{1}, \textbf{Patrick Vossler}\textsuperscript{1,$\dagger$}, \textbf{Jean Feng}\textsuperscript{1,$\dagger$}}
\aistatsaddress{\textsuperscript{1}University of California, San Francisco \quad \textsuperscript{2}University of Maryland \quad \textsuperscript{3}Johns Hopkins University} ]

\begin{abstract}
Continuous monitoring of generative AI systems requires detecting distribution shifts in the unstructured data they ingest and generate.
Existing monitoring methods typically control false alarms by requiring the discrepancy measure used to detect shifts to be \textit{pre-specified}.
However, for unstructured data, pre-specification creates a sensitivity--breadth trade-off: a narrowly defined measure is sensitive to the anticipated shift but may miss unexpected changes, whereas a broadly defined measure may detect a wider range of shifts but with reduced sensitivity.
Although adaptively hypothesizing the shift direction can overcome this limitation, such data-dependent adaptation generally invalidates conventional false-alarm guarantees.
We reconcile adaptivity with statistical validity through testing-by-betting, reframing sequential changepoint detection as a betting game in which the bettor continually updates its hypothesized shift direction, or witness function, as new observations arrive.
We introduce \textit{asymptotic e-detectors}, which approximate the pre-change distribution using a calibration set of size $m$, and prove that they asymptotically control the average run length to a false alarm as $m$ increases.
Empirically, calibration sets as small as $m=30$ provide effective control.
We then develop adaptive monitoring strategies ranging from black-box bettors that learn from embeddings to interpretable bettors based on LLM-proposed concepts.
Across three real-world datasets, including data from a clinical AI scribe deployment, adaptive monitors detect shifts faster than fixed monitors at the same false-alarm level.
Further, when the true shifts are sparse and interpretable, LLM-concept monitors provide the fastest and most interpretable detection, while also guiding root-cause diagnosis for complex shifts.
\end{abstract}

\section{Introduction}
\label{sec:intro}
\renewcommand{\thefootnote}{$\dagger$}%
\footnotetext{Co-senior authors.}%
\renewcommand{\thefootnote}{\arabic{footnote}}%

Preventing performance decay in deployed generative AI systems requires monitoring throughout the entire pipeline, including the input distribution, user behavior, and model outputs.
Existing monitoring procedures follow the general recipe of pre-specifying a discrepancy measure and tracking shifts with respect to this metric.
Pre-specification plays a critical role \citep{Page1954-vy, Shiryaev1963-xq, Tartakovsky2014-yq}, as it ensures control over false alarms; this is typically measured by Average Run Length under the null hypothesis of there being no changepoint ($\ARL$),
which quantifies the ``lifespan'' of a monitor until it must be reset.
For instance, LLM-based monitoring methods track scores with respect to a fixed performance rubric \citep{Husain2026-pt} or shifts in fixed directions away from the original data distribution \citep{Zhu2024-yl, Liu2024-qd, Xu2025-em}.
Other methods convert unstructured data into high-dimensional vectors, which can then be monitored for shifts with respect to pre-specified directions \citep{Cao2024-je, Chen2024-bj, Rabanser2019-vc} or kernel functions \citep{Harchaoui2009-or, Chang2018-kf}.
Monitoring metrics should be specified with immense care, as performance of monitoring procedures is highly sensitive to this choice \citep{Feng2025-rq}: strong alignment between the pre-specified measure and the future shift direction leads to fast detection, while poor alignment may lead to having little to no detection power.

Nevertheless, pre-specification has key limitations that leave it misaligned with modern AI practice.
First, it is typically difficult to pre-specify the monitoring metric, since future shift directions are, by definition, unknown.
This forces monitoring teams to choose between concentrating power on detecting shift in particular directions (sensitivity) versus spreading power across many possible directions (breadth).
Defining shifts for high-dimensional or unstructured data is particularly difficult given that such data can experience a wide variety of shifts.
As such, recent works on retrospective testing have instead sought to \textit{learn} so-called ``witness functions'' that highlight differences between data distributions by training a classifier \citep{Lopez-Paz2017-uc, Kubler2022-up, Kubler2022-df}. Second, pre-specification prevents the monitoring metric from adapting as information accumulates, whereas a monitoring team may gain a better understanding of the data and likely shift directions as data accumulates.
For instance, a shift may become obvious to a human after a small number of observations \citep{Faust2024-ug}.
In fact, the rigidity of current statistical monitoring procedures has led monitoring to be viewed as a fundamentally ``human'' activity \citep{Montgomery2013-sk}, in which statistical evidence complements, rather than replaces, human judgment.
Finally, with the increasing reasoning capabilities of LLMs, monitoring practitioners are increasingly drawn to using LLMs to assist in monitoring \citep{Almodovar2024-ap}, but formal statistical monitoring procedures still lack the ability to incorporate such adaptations on-the-fly.

As such, the goal of this paper is to answer the following question: \textit{can we monitor unstructured data streams by adaptively increasing sensitivity to shifts learned from past data, while still maintaining statistical guarantees over false alarms?}
To this end, we look to the Safe Anytime-Valid Inference (SAVI) framework \citep{ramdas2023game}, which enables nearly arbitrary adaptations during the testing process by reframing sequential testing as a game of betting against the null hypothesis.
This game must satisfy one simple martingale-like property: the game's returns should be centered so that any bettor---however adaptive and adversarial---cannot gain wealth in expectation under the null hypothesis \citep{Shafer2021-ni, Grunwald2024-qs, Ramdas2025-uv}.
A bettor's current wealth is thus the current level of evidence against the null, and rejection of the null is warranted when the bettor's wealth exceeds some threshold; this wealth is called an ``e-value,'' as the analog of the classical p-value.

In sequential changepoint detection, the null hypothesis is that there is no change, i.e., $H_0: P_t = P \quad \forall t = 1,2,\cdots$, while the alternative is that the data switches to a post-change distribution $Q$ after some changepoint $\tau$.
To control $\ARL$, recent work has built on e-values and the SAVI framework to introduce ``e-detectors'' as a broad class of procedures \citep{Shin2024-dj}.
Through the betting lens, these methods can be generally viewed as playing the following game:
\begin{tcolorbox}[
    colback=white,
    colframe=black,
    title=\texttt{Changepoint-Detection-Game},
    fonttitle=\bfseries,
    boxsep=1pt,      left=2pt,        right=2pt,       top=2pt,         bottom=2pt,      toptitle=2pt,    bottomtitle=2pt  ]
At each time $t$:
\begin{enumerate}[leftmargin=*, itemsep=0pt, topsep=2pt, partopsep=0pt, parsep=0pt]
\item Based on data up to time $t-1$, the bettor proposes a witness function $\hat{f}_t$, whose values are expected to be large post-change but small pre-change, and a bet scaling factor $\lambda_t$.
    \item The game's return at time $t$ is the scaled difference between the value of $\hat{f}_t$ on the new observation $X_t$ and its expectation in the pre-change distribution.
\end{enumerate}
\end{tcolorbox}

\textit{Playing the game} is thus quite simple: at each time point, the bettor adaptively hypothesizes when/what changed in the stream and designs a witness function to detect this.
Indeed, humans may actually be quite adept at this game, as a human who is manually monitoring a stream of text or images  will naturally hypothesize shift directions (e.g., images become more blurry) and check alignment of incoming observations against these candidate directions.

\begin{figure*}[t]
    \centering
    \vspace{-0.5cm}
    \includegraphics[width=0.85\linewidth]{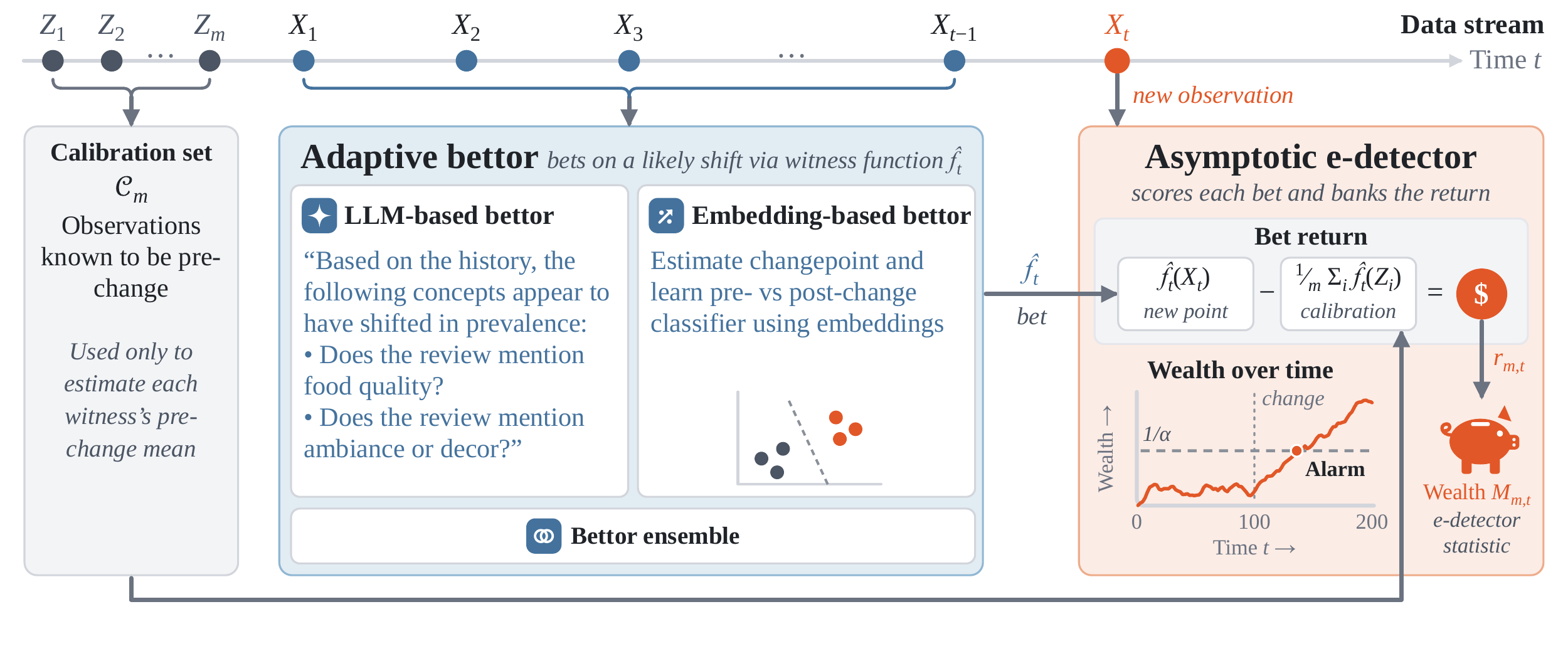}
    \vspace{-0.8cm}
    \caption{Adaptive monitoring procedure overview: At each time step, an adaptive bettor uses past observations to bet on how the data may have shifted. When the next observation arrives, the bet's return is computed by comparing it with a calibration set, and an alarm is raised once the accumulated wealth crosses a threshold.}
    \vspace{-0.5cm}
    \label{fig:overview}
\end{figure*}

On the other hand, \textit{setting up the game} is far from straightforward, because
the exact expectation of the witness function is unknown in practice; so it is difficult to center the game's returns to satisfy the martingale requirement and guarantee there is no winning strategy under the null.
To our knowledge there are currently no highly adaptive changepoint detection methods for the most basic setting of monitoring a single stream of unstructured data.
Instead, prior works have mainly explored variants of this game, such as monitoring paired data streams where the paired structure simplifies the design of game returns \citep{Shekhar2024-xr, Podkopaev2023-yh}.
However, one typically does not have access to a second parallel data stream representing the relevant pre-change distribution (e.g., an AI scribe just deployed in a hospital has, by definition, limited pre-change data).

To this end, this work introduces \textit{asymptotic e-detectors} to adaptively monitor a change in a single, potentially unstructured, data stream.
Asymptotic e-detectors play an approximation of the \texttt{Changepoint-Detection-Game}, where the pre-change distribution is approximated by a finite calibration dataset.
Following the foundational work of \citet{Chu1996-fg}, we show that this game controls the $\ARL$ asymptotically, as the size of the calibration dataset $m \rightarrow \infty$; empirically, we find that $\ARL$ is well-controlled even for calibration datasets as small as $m=30$.
Having set up the game, we then introduce various betting strategies: bettors that detect black-box shifts using embeddings, bettors that detect interpretable shifts by iteratively searching over candidate concepts using LLMs, and bettor ensembles.
Across three real-world text/image data streams, we find that adaptive monitoring procedures outperform their fixed counterparts.
When the true shift is sparse and interpretable, using LLMs to learn which concepts to monitor can not only recover the shifted concepts but also substantially accelerate detection speed.
Code for reproducing the results in this work is available at \url{https://github.com/jjfenglab/monitoring-by-betting}.

\section{Related work}
\label{sec:related}

\textbf{Monitoring LLMs and GenAI in practice.}
The most widespread practice for monitoring drifts in LLM applications involves measuring performance of the system outputs in response to system inputs, often approximated using LLM-as-a-judge \citep{Databricks2024-rb, Croxford2025-wl, Husain2026-pt, Vossler2026-ii}.
However, these approaches require the metrics to be pre-specified and thus suffer from the sensitivity--breadth trade-off, such as failing to include or poor alignment with actual drift directions \citep{zhong2023goal, Shankar2024-od, Dunlap2024-td}.
The methods developed in this paper can thus provide additional flexibility for monitoring performance or detecting more general shifts.
Some works have used e-values to test for the correctness of LLM outputs \citep{dhillon2025scores, sadhuka2025valuator, Zhou2026-hn}, but not for monitoring unstructured data for changepoints.
We note that other specialized testing approaches designed solely to test shifts in LLM behavior---such as repeatedly using the same test set \citep{Chen2024-fg} or actively constructing prompts for testing \citep{Zhu2025-kb, chauvin2026log, Chauvin2026-dd}---are outside the scope of this work.

\textbf{Changepoint detection procedures.}
There is a vast literature on CUSUM or Shiryaev--Roberts (SR) procedures for changepoint detection, from classical (generalized) likelihood ratio procedures \citep{Siegmund1995-to} to more recent classifier two-sample test (C2ST)-based ones \citep{Puchkin2023-yt, Gong2022-dz}.
However, these classic methods often simulate or analytically derive the false alarm-controlling threshold by assuming a simple, parametric structure on the pre-change distribution or requiring the pre-change distribution to be known, which is not practical for monitoring unstructured inputs and outputs of GenAI systems.
Recently, e-detectors \citep{Shin2024-dj}, which generalize classic changepoint detection procedures to composite and nonparametric settings, have become increasingly popular with the benefit that thresholding at $1/\alpha$ guarantees that $\ARL$ is at least $1/\alpha$.
E-detectors can be constructed as CUSUM and SR procedures based on e-processes, which generalize the notion of test (super)martingale. CUSUM and SR variants of conformal testing martingales (CTMs) \citep{Vovk2021-re, Prinster2025-bv} can thus also be viewed as e-detectors on a filtration of conformal p-values.
Compared with standard CTMs which conventionally assume a pre-specified conformity score, e-detectors allow all quantities entering the betting function to be adaptive provided they are predictable, including post-change estimates, betting sizes, and the witness function \citep[among others]{Howard2020-bj, Ramdas2020-wr, Shekhar2024-xr, Richter2025-bk}.
Weighted CTMs \citep{Prinster2025-bv} allow for monitoring an adaptive null hypothesis, but here we focus on adaptively learning the alternative hypothesis.
Nevertheless, existing e-detector methods are restricted to data streams with specific structures that are amenable to martingale-like requirements.
To our knowledge, there are no e-detector methods for the most common and basic task encountered by GenAI systems: monitoring a single, unstructured data stream.
To address this technical gap, this work introduces an extension of the e-detector framework that relaxes the strict martingale requirements.

\textbf{Interpretable explanations of shifts.}
Existing work on interpreting data shifts primarily focuses on providing post-hoc explanations \citep{Hinder2023-jv, Hinder2024-nm, Rabinovich-ra}, which can build trust in AI monitoring systems \citep{Lipton2018-wu, Doshi-Velez2017-yv} and facilitate the diagnosis and correction of the underlying model \citep{Caruana2015-ul}.
However, post-hoc explanations may not be faithful to the shift detected originally using black-box methods \citep{Rudin2019-hw}.
Work on white-box changepoint detection methods is extremely limited \citep{Wijaya2021-qh} and requires the set of potential shift directions to be pre-specified.
White-box anomaly/Out-of-Distribution (OOD) detection takes a similar approach, using natural language to pre-specify the likely distribution/definition of an OOD sample
\citep{Esmaeilpour2022-vj, Cao2024-je, Dai2023-ic, Zhu2024-yl}.
However, pre-specification carries a major risk of missing the actual shift, as future shifts are typically unknown in practice.
As such, this work considers learning interpretable directions along which to detect shifts, which draws on earlier works learning natural language explanations to classify between two data distributions \citep{zhong2022describing, zhong2023goal, dunlap2024describing} and searching across the infinite space of concepts to maximize the performance of a Concept Bottleneck Model  \citep{Koh2020-kl, Feng2025-qd}.

\section{The Monitoring Game}
\label{sec:framework}

Suppose we sequentially observe independent random variables $(X_1, X_2, \dots)$ with $X_t \in \mathcal{X}$ for all $t \geq 1$, whose joint distribution is denoted by $\mathbf{P}$.
Let $(\Fil_{t})_{t\geq 0} $ be the natural filtration with $\Fil_t = \sigma(X_1, \dots, X_t)$.
Nonparametric sequential changepoint detection can be formally defined as monitoring the null hypothesis of there being no changepoint and the alternative of there being a single changepoint, i.e.,
\begin{align}
    H_0&: \exists \ P, \ \text{s.t.} \ \mathbf{P} = P^{\infty}, \label{eq:nullHypothesis} \\
    H_1&: \exists \tau^* \geq 1 \text{ and } Q \ne P \text{ s.t. } \mathbf{P} = P^{\tau^*} \times Q^{\infty},
    \label{eq:AltHypothesis}
\end{align}
where $P$ and $Q$ are unknown.
For a monitoring procedure with stopping time $\tau$, we characterize its operating characteristics through its $\ARL$ and Expected Detection Delay (EDD), which are analogs of Type I and II error in classical hypothesis testing and formally defined as
\begin{align}
    \ARL &= \E[\tau], \\
    \EDD &= \mathbb{E}_{H_1}\left[\tau - \tau^* | \tau > \tau^* \right],
\end{align}
where $\E$ and $\mathbb{E}_{H_1}$ denote expectation under $P^\infty$ and $ P^{\tau^*} \times Q^{\infty}$, respectively. Drawing on the e-detector framework, the key to controlling $\ARL$ is to design a proper \texttt{Changepoint-Detection-Game} that ensures there is no winning strategy under the null.
To this end, we begin with describing e-detectors for playing an idealized form of the game and then how their extension---\textit{asymptotic} e-detectors---can play the approximate game in real-world settings.
We defer all proofs to Appendix~\ref{app:proofs}.

\subsection{E-detectors for the idealized game}

E-detectors are a general class of sequential tests for the stated changepoint detection hypotheses in the following sense:
\begin{definition}[E-detector]
    Let $M = (M_t)_{t\geq 0} $ denote a nonnegative adapted process with respect to a filtration $\mathcal{F} = (\mathcal{F}_t)_{t \geq 0}$.
The process $M$ is an e-detector for $H_0$ if
    $
        \mathbb{E}_{H_0} [M_\tau] \leq \mathbb{E}_{H_0} [\tau]
    $ for all $\tau \in \mathcal{T}$,
     where $\mathcal{T}$ is the set of all stopping times with respect to $\mathcal{F}$.
\end{definition}
From this definition, it immediately follows that for any stopping time $\tau = \inf\{t\geq1:M_t \geq 1/\alpha\}$, the $\ARL$ is guaranteed to be no smaller than $1/\alpha$.

Despite this abstract definition, e-detectors constructed from a game-theoretic setup can be defined quite concretely as the aggregate wealth of a bettor playing the \texttt{Changepoint-Detection-Game}.
In particular, suppose there is a sequence of subgames that are smaller versions of the \texttt{Changepoint-Detection-Game} and begin at each timepoint with 1 unit of wealth.
For each of these subgames, we define a corresponding e$_j$-process that represents the bettor's wealth in the subgame starting from candidate changepoint $j$, i.e.,
\begin{definition}[e$_j$-process]
\label{def:e_process}
    For any $j \geq 1$, $\Ej = (\Ej_t)_{t\geq 0} $ is an e$_j$-process for $H_0$ if it is a nonnegative adapted process satisfying $\Ej_0 = \dots = \Ej_{j-1} = 1$, and
    \begin{align*}
        \mathbb{E}_{H_0} \left[ \Ej_\tau | \mathcal{F}_{j-1} \right] \leq 1, \forall \tau \in \mathcal{T}.
    \end{align*}
\end{definition}
In words, this definition states that each e$_j$-process represents a subgame where there is no way that the bettor can expect to accumulate wealth under the null.
One can then show that valid e-detectors can be constructed through various aggregations of wealth across subgames: based on the classical Shiryaev--Roberts (SR) procedure \citep{Shiryaev1963-xq, Roberts1966-eh}, the SR e-detector is defined as the total wealth across all subgames.
Based on \citet{Page1954-vy}, the CUSUM e-detector is defined as the maximum wealth across all these subgames.

To obtain valid e-processes, the key question is then how to design this betting game.
For each unit of wealth, an aggressive bettor could bet all their wealth in hope of a large return rate $r_t$, but this strategy runs the risk of bankruptcy.
In practice, we require each bettor provide a (positive and predictable) betting fraction $\lambda_t$ that is small enough to avoid betting \textit{all} their wealth; that is, to bound the return rate above zero, $\ 1+\lambda_t r_t > 0$, for all $t$ (details in Section \ref{sec:asymptotic_e_detector_approx_game}).
Next, suppose for simplicity that the rate of return $r_t$ is the same across all active subgames.
To ensure there is no winning under the null, we require the expected return rate not to exceed zero under the null, i.e.,
\begin{assumption}
    The return process $r = (r_t)_{t \ge 1}$ adapted to filtration $\mathcal{F}$ satisfies $\E\left[r_t | \mathcal{F}_{t-1}\right] \le 0$ for all $t$.
    \label{ass:martingale}
\end{assumption}
Consequently, the $j$-th e-process  is the cumulative product $E_t^{(j)} := \prod_{s=j}^t (1 + \lambda_s r_s)$, which is a supermartingale under $H_0$.
The SR e-detector $(M_t^{\SR} )_{t\geq0} = (\sum_{j=1}^t \Ej_t)_{t\geq0} $ can be calculated with the recursion
\begin{align}
    M_t^{\SR}  = (1 + M_{t-1}^{\SR}) (1 + \lambda_t r_t) \text{ with } M_0^{\SR}=0;
\end{align}
the CUSUM e-detector can be derived using a similar recursion.

We can now set up the \texttt{Changepoint-Detection-Game} as follows.
Let $\hat{f}_t : \mathcal{X} \to \mathbb{R}$ be a $\Fil_{t-1}$-measurable witness function selected by the bettor, which can be determined however adaptively using any or all of the past data.
In an ideal world where the true pre-change distribution $P$ is known, it is natural to define the game's return as the centered quantity
\begin{align}
    r_t = \hat{f}_t(X_t) - \E[\hat{f}_t(X_t) \mid \mathcal{F}_{t-1}]
    \label{eq:return_ideal}
\end{align}
to satisfy Assumption~\ref{ass:martingale}.
Then under this ideal game design, the statistics $M^{\SR}$ and $M^{\CU}$ are valid e-detectors.
However, without access to the true pre-change distribution, we cannot obtain the exact value of $\E[\hat{f}_t(X_t) \mid \mathcal{F}_{t-1}]$ and thus cannot form a valid e-detector.
The question is then how we can still design monitoring processes that control the $\ARL$.

\subsection{Asymptotic e-detectors for the approximate game}
\label{sec:asymptotic_e_detector_approx_game}

As the idealized \texttt{Changepoint-Detection-Game} is not feasible in practice, an alternate approach is needed.
To this end, we draw on the asymptotic changepoint detection framework introduced in \citep{Chu1996-fg}, which introduced the notion of a calibration dataset and a new asymptotic framework defined with respect to the size of the calibration dataset; this framework foundationally shifted the changepoint detection field at the time, as it drastically simplified approaches to obtaining (asymptotic) control of the $\ARL$.
Inspired by this work, we introduce a similar asymptotic framework.

Suppose we have access to a calibration set $\mathcal{C}_m = \{Z_1, \dots, Z_m\}$ of size $m$ sampled IID under the pre-change distribution $P$. Intuitively, we could estimate $\E[\hat{f}_t(X)\mid \mathcal{F}_{t-1}]$ with the empirical mean over $\mathcal{C}_m$ and thus approximate the idealized return $r_t$ with
\begin{align}
r_{m,t} &= \hat{f}_t(X_t) - \frac{1}{m}\sum_{i=1}^{m} \hat{f}_t(Z_i).
\label{eq:return_approx}
\end{align}
Plugging $r_{m,t}$ in place of $r_t$ of an SR e-detector, we define
\begin{align}
    M_{m,t} &= (1 + M_{m,t-1}) (1 + \lambda_t r_{m,t}) \label{eq:M_mt}
\end{align}
for $t\geq 1$ and with $M_{m,0} = 0$.
Intuitively, as $m \rightarrow\infty$, we expect the conditional expectation of $r_{m,t}$ to converge to zero, and thus $M_{m,t}$ to behave like an e-detector, under one critical requirement: the witness function  must be fully independent of the calibration dataset, so that the empirical average in \eqref{eq:return_approx} is a uniformly consistent estimate of the true expectation under the null.

Motivated by this, we introduce the notion of an \emph{asymptotic e-detector} as follows.
\begin{definition}[Asymptotic e-detector]
\label{def:asymptotic-edetector}
For each calibration size $m$, let $\mathbf{M}_m = (M_{m,t})_{t \geq 1}$ be a nonnegative process adapted to $\mathcal{G}_m = (\mathcal{G}_{m, t})_{t\geq 0} $ with $\mathcal{G}_{m, t} = \sigma(\mathcal{F}_t, \mathcal{C}_m)$. Then $\mathbf{M}_m$ is an \emph{asymptotic e-detector} for $H_0$ if, for the stopping time $\tau_m = \inf\{t : M_{m,t} \geq 1/\alpha\}$,
\begin{equation}
\limsup_{m \to \infty} \left( \mathbb{E}_{H_0}[M_{m,\tau_m}] - \mathbb{E}_{H_0}[\tau_m] \right) \leq 0.
\label{eq:asymptotic-edetector}
\end{equation}
\end{definition}

Let us further suppose the following regularity conditions are satisfied:
\begin{assumption}
\label{ass:regularity}
(i) For all $t \geq 1$, witness functions $\hat{f}_t$ are measurable w.r.t. $\mathcal{F}_{t-1}$ and take values in $[-B,B]$.
(ii) All witness functions $\{\hat{f}_t\}_t$ lie in a fixed $P$-Glivenko-Cantelli class $\mathbb{F}$.
(iii) The calibration set $\mathcal{C}_m = \{Z_1, \dots, Z_m\} \overset{\text{iid}}{\sim} P$ and independent of the monitoring stream $X_1, X_2, \dots$.
(iv) The stopping times $\{\tau_m\}_m$ under $H_0$ are uniformly integrable.
\end{assumption}

Condition~(i) states that the witness function cannot depend on the calibration dataset or any downstream variable of the calibration data (e.g., the game's return $r_t$) and the boundedness is satisfied by standard classifier classes on fixed representations.
Condition~(ii) is a standard complexity assumption on the witness function class.
A calibration dataset that satisfies Condition~(iii) is often easy to collect in practice.
Condition~(iv) is a standard regularity assumption; if instead the expected stopping time is unbounded, the $\ARL$ trivially exceeds the nominal level.

Given the assumptions above, the following theorem formally states that an asymptotic e-detector that reuses the same calibration set of size $m$ over time remains valid asymptotically as $m\rightarrow \infty$.
\begin{theorem}
\label{thm:main}
Consider the detector $(M_{m,t})_{t \geq 1}$ defined by~\eqref{eq:M_mt} with $\lambda_t \leq 1/(2B)$.
Define the stopping time $\tau_m = \inf\{t\geq1:M_{m,t} \geq 1/\alpha\}$.
Under Assumption \ref{ass:regularity}, $(M_{m,t})_{t \geq 1}$ is an asymptotic e-detector.
Consequently, for any $\alpha \in (0,1)$, $\liminf_{m \to \infty} \E[\tau_m] \geq 1/\alpha$, i.e., the $\ARL$ is asymptotically controlled at level $\alpha$.
\end{theorem}
Finally, we note that while Theorem~\ref{thm:main} only characterizes the asymptotic $\ARL$ and not the convergence rate itself, we intuitively expect $M_{m,t}$ to converge quite quickly to the idealized e-detector $M_t$, as martingale increments are defined in terms of conditional expectations and expectations generally converge quickly (in contrast to convergence of conditional probability distributions).

\section{Betting strategies}
\label{sec:bettor}

Having set up the game, the question is what betting strategy one should employ, i.e., how should one select the witness function in Step 1 of \texttt{Changepoint-Detection-Game}?
To this end, we extend ideas from Classifier-based Two-Sample Testing (C2ST) \citep{Lopez-Paz2017-uc, Shekhar2024-xr}, which tests for differences between two data distributions by training a classifier to distinguish between them.
Compared to classical two-sample testing methods that use pre-specified distance measures, this classification framing is particularly advantageous for unstructured data, as pre-specifying distance for such data modalities is difficult but AI/ML models are effective at learning discriminators for this space.

Consequently, we introduce Classifier-based Betting, in which the bettor adaptively selects/learns classifiers to distinguish pre- and post-change data.
Each betting step can be broken down into four decisions:
\begin{tcolorbox}[
    colback=white,
    colframe=black,
    title=Classifier-based Betting,
    fonttitle=\bfseries,
    boxsep=1pt,      left=2pt,        right=2pt,       top=2pt,         bottom=2pt,      toptitle=2pt,    bottomtitle=2pt  ]
\begin{enumerate}[label=\alph*),leftmargin=*, itemsep=0pt, topsep=2pt, partopsep=0pt, parsep=0pt]
    \item Select data representation $\phi_t: \mathcal{X} \mapsto \mathbb{R}^K$.
    \item Estimate changepoint $\hat{\tau}_t$.
\item Fit a classifier $\hat{p}_t$ between the pre- vs post-change groups with features $\phi_t(X_i)$ and pseudo-labels $Y_i = \mathbbm{1}\{i \ge \hat{\tau}_t\}$. Define witness $\hat{f}_t = \hat{p}_t \circ \phi_t$.
    \item Select scaling factor $\lambda_t$.
\end{enumerate}
\end{tcolorbox}

This work focuses on steps (a)-(c); for simplicity, this work fixes $\lambda_t$ to be constant and we refer readers to other papers that study adaptive selection of $\lambda_t$ \citep{Shekhar2024-xr}.
We present two general betting approaches that fork at the first decision point of step (a): using embedding representations versus interpretable extractions of the data.
Natural choices for steps (b) and (c) emerge after this first decision point.
Finally, since step (a) is arguably the biggest factor influencing detection speed, ensembling can be used to hedge against the uncertainty of this choice (note that ensembling preserves validity of the (asymptotic) e-detector); we evaluate ensembling later in Section~\ref{sec:experiments}.

\subsection{Bettor for black-box shifts}
We first consider the class of betting strategies that detects shifts in embedding space, which is useful when one has minimal prior knowledge of the likely shift type and/or is interested in detecting \textit{any} shift rather than a restricted set of shifts.
As such, let $\phi_{emb}: \mathcal{X} \mapsto \mathbb{R}^p$ denote the embedding model selected in step (a).
For steps (b) and (c), one can then draw on the large family of general-purpose changepoint learning procedures \citep{Gretton2012-ws, Bonneel2015-ib} and classification procedures.
A simple example---which we use in later experiments---is to (b) find the maximizer $\hat{\tau}_{t}$ of the classical changepoint test statistic for multivariate mean shifts \citep{Basseville1993-ka, Bai2010-pd}, i.e., \begin{equation}
\small
\max_{\tau} \frac{(\tau-1)(t-\tau+1)}{t}  \left\|\overline{\phi_{emb}(X)}_{1:(\tau-1)} - \overline{\phi_{emb}(X)}_{\tau:t}\right\|^2
\label{eqn:cp_meanShift}
\end{equation}
where $\overline{\phi_{emb}(X)}_{a:b}$ represents the mean extraction over positions $a$ to $b$ and $\|\cdot \|$ the $L_2$ norm, and then (c) fit a logistic regression model for the learned split.

\subsection{Bettor for interpretable shifts}

\begin{figure}
    \centering
    \includegraphics[width=\linewidth]{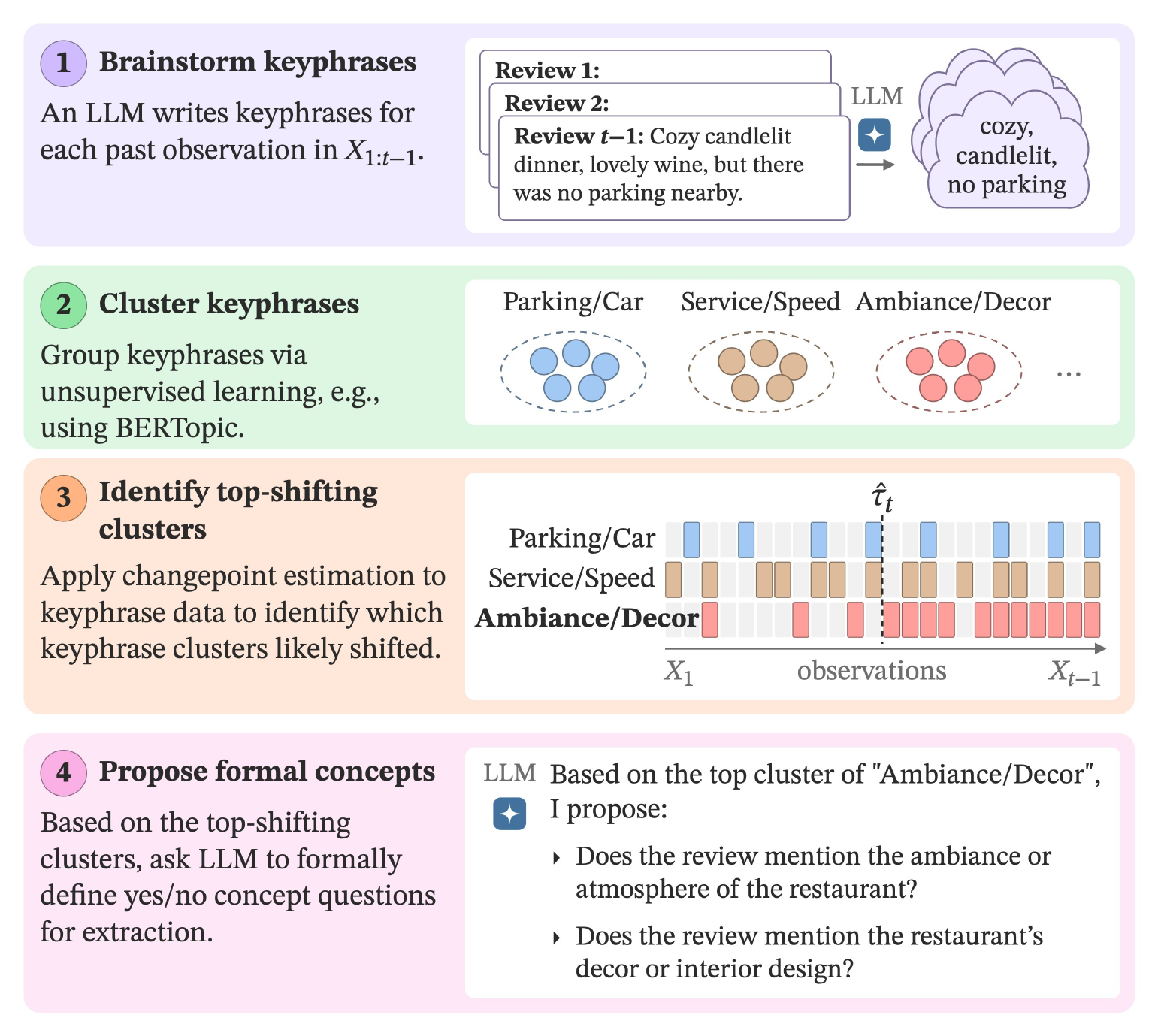}
    \vspace{-1cm}
    \caption{\llmconcept{} bettor for estimating when and what shifted, using LLM reasoning, unsupervised learning, and statistical changepoint estimation.
    }
    \vspace{-0.3cm}
    \label{fig:concept_proposer}
\end{figure}

Next consider a bettor who relies on interpretable bets, which is useful when they have domain knowledge but difficulty pre-specifying the exact set of possible shifts upfront.
Restricting to interpretable bets has two benefits: (i) constraining the set of shifts to those of practical interest (e.g., interested in shifts in clinical note quality but not writing style) and thereby reducing false alarms and (ii) injecting prior knowledge to accelerate detection, compared to black-box bets.

Formally, the interpretable bettor must supply a concept-extractor function $\phi_{\texttt{C}_t}: \mathcal{X} \mapsto \mathbb{R}^K$ in step (a), where $\texttt{C}_t$ is an adaptively selected set of concepts.
While concept extraction is generally straightforward (e.g., using an LLM \citep{Kojima2022-eo, Agrawal2022-me}), there is limited work on selecting concepts $\texttt{C}_t$ to optimize changepoint detection, as detailed in Section~\ref{sec:related}.
Figure~\ref{fig:concept_proposer} outlines an \llmconcept{} bettor that extends concept selection in \citep{Feng2025-qd} to the changepoint setting, in which the AI agent uses statistical tools to identify candidate changepoints and concepts; we found this pipeline to work well in practice and leave further refinement of this pipeline to future work.
More specifically, we have an LLM agent first brainstorm potential shifts by generating a ``word cloud'' of keyphrases describing each observation up to the current time point, cluster keyphrases using BERTopic \citep{Grootendorst2022-jl}, estimate a working changepoint and top-changing clusters based on keyphrases \citep{Enikeeva2019-yu}, and have an LLM propose formal yes/no concepts based on the top clusters  \citep{Feng2025-qd}.
More details are in the Appendix.

For steps (b) and (c), one can use any changepoint detection and classification methods but we suggest leveraging the unique advantage of the interpretable bettor: because concept directions are generally disentangled (as opposed to dense embeddings), this bettor can uniquely detect sparse shifts.
To this end, we pair interpretable concept bettors with changepoint detection methods for sparse alternatives \citep{Enikeeva2019-yu}, which jointly learn when and which subset of features shifted (up to a maximum number of $K$ features) by finding the maximizer $\hat{\tau}^{(sparse)}_{t}$ of the statistic 
\begin{equation}
\hspace{-1cm}
\footnotesize
\begin{aligned}
    \max_{\tau, m \in \mathcal{S}_K}\frac{(\tau-1)(t-\tau+1)}{t} \left\|\overline{\phi_{\texttt{C}_t}(X)_{m,1:(\tau-1)}} - \overline{\phi_{\texttt{C}_t}(X)_{m,\tau:t}}\right\|^2 - |m|
\end{aligned}
\label{eqn:cp_sparse}
\end{equation}
where $\mathcal{S}_K$ represents all possible subsets of $\{1,\cdots,p\}$ of sizes $1$ to $K$, and $\overline{\phi_{\texttt{C}_t}(X)_{m,a:b}}$ represents the mean of features restricted to subset $m$ over positions $a$ to $b$.
Likewise, we can use an elastic net penalty when fitting the logistic regression model.

\section{Experiments}
\label{sec:experiments}

We evaluate the adaptive monitoring-by-betting framework across three experimental settings: a text stream with an injected topical shift (Section~\ref{sec:exp_cebab}), an image stream from bird photographs with an injected image-attribute shift (Section~\ref{sec:exp_cub}), and real-world clinical AI scribe data (Section~\ref{sec:exp_notes}).
We implement three adaptive bettors that differ in their representation $\phi_t$, namely raw embeddings (\embeddingarm), sparse autoencoder features of the embeddings (\saearm), and concepts proposed by \texttt{gpt-4.1-mini} (\llmconcept), as well as an \ensemblearm{} that averages the \embeddingarm{} and \llmconcept{} detectors.
We report the truncated $\EDD$, the mean number of post-change observations until an alarm or the end of the stream, as a function of the nominal $\ARL$.

For fixed monitors, we run SR on CTMs \citep{Vovk2021-re} with fixed (non)conformity scores that trade breadth for sensitivity.\footnote{We omit asymptotic e-detectors with a fixed witness, as they are generally more conservative than SR for CTMs with fixed scores, as CTMs provide tight $\ARL$ control.}
\centroidarm{} scores the distance to the centroid of the calibration set, while \informedshift{} and \unrelatedshift{} score the cosine similarity with a text-defined direction for the true shift or an unrelated one.
As a performance ceiling, \oraclearm{} uses the difference between the mean post- and pre-change embeddings as its direction.

All experiments are run for 30 replicates for all LLM-based detectors for cost purposes and $100$ replicates for all others.
All methods use a calibration set of $m = 30$ based on results in Appendix~\ref{app:vary_m}, which verify $\ARL$ control across a large range of $m$ values.
See Appendix~\ref{app:experiments} for full experiment details.

\subsection{Monitoring restaurant reviews in CEBaB}
\label{sec:exp_cebab}

We first evaluate how quickly adaptive monitoring procedures can detect a topical shift in text.
Using the real-world CEBaB restaurant review dataset \citep{Abraham2022-om}, we construct disjoint pre- and post-change pools in which reviews mentioning ambiance are strongly upweighted post-change, so that the rate of ambiance mentions increases from $0.52$ to $0.88$.
The \informedshift{} and \unrelatedshift{} directions use texts about ambiance and about parking (Table~\ref{tab:ctm_exemplars}).

All procedures with null streams are well calibrated, with simultaneous $95\%$ bands for their empirical $\ARL$ that contain or lie above the desired level (Appendix Fig~\ref{fig:cebab_calibration}).
Compared to fixed monitors, adaptive monitoring speeds up detection (Fig~\ref{fig:results}a), with \llmconcept{} achieving the lowest EDD.
\llmconcept{} quickly learns the correct concept, defined through yes/no questions like ``Does the review mention ambiance or atmosphere of the restaurant?'' (Fig~\ref{fig:results}e); this illustrates how it increases sensitivity to the true shift and explains why \llmconcept{} is only slightly slower than the \oraclearm{} detector.
Meanwhile, the adaptive \ensemblearm{} achieves the second fastest EDD, by hedging across both the interpretable and black-box shifts.
Finally, the slow speed of \informedshift{} is perhaps surprising since it already knows the true concept shift.
However, it shows that drafting text to represent candidate drift directions in embedding space can be quite difficult in practice.

\vspace{-0.4cm}
\subsection{Monitoring bird images from CUB}
\label{sec:exp_cub}

\begin{figure*}[t]
\vspace{-0.5cm}
  \centering
  \includegraphics[width=0.8\textwidth]{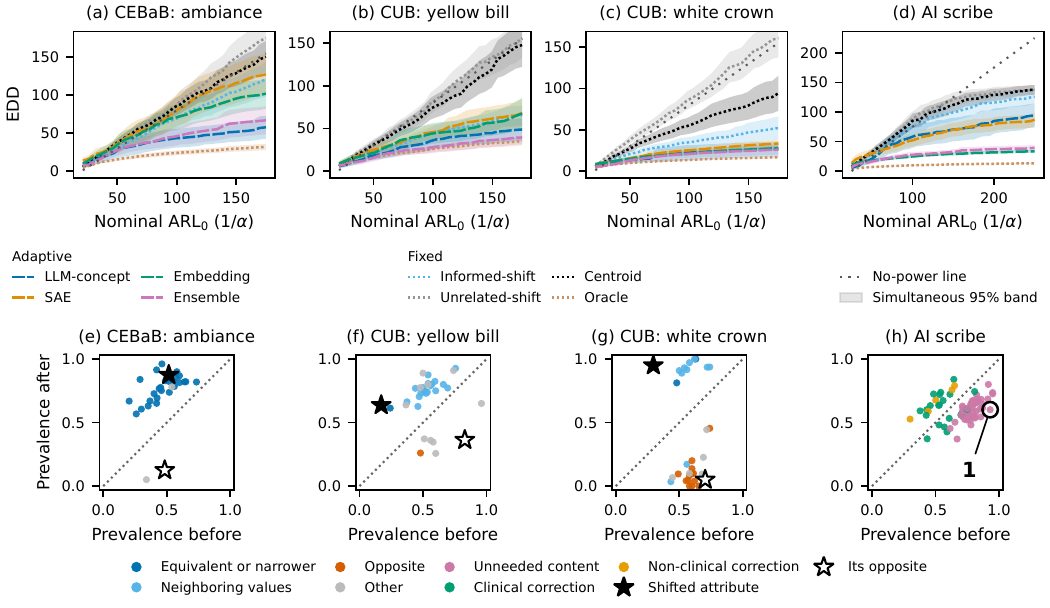}
  \vspace{-0.6cm}
  \caption{Experimental results.
  Top: EDD against nominal $\ARL$, where lower is better.
  Bottom: Concepts learned by the \llmconcept{} bettor, plotted by prevalence before vs after the changepoint.
  Concepts in (e-g) are colored by their agreement with the true shifted attribute (stars).
  Concepts in (h) are colored by themes (Table~\ref{tab:scribe_themes}), as true shifts are unknown for real-world AI deployments.
  The circled concept in (h) asks whether the draft includes unneeded non-clinical information, such as a review of systems or introductory sentences.
  }
  \label{fig:results}
\end{figure*}

Next we evaluate the proposed method on detecting shifts in images.
Using bird photographs from CUB-200-2011 \citep{Wah2011-kc}, we construct a set of streams that upsamples birds with yellow bills post-changepoint (prevalence of $0.17$ to $0.64$) and an analogous set but for white crowns (prevalence of $0.30$ to $0.95$).
We use CLIP embeddings to define the pre-specified directions for the fixed monitors; for instance, the \informedshift{} direction for streams with a shift in yellow bills uses the exemplar text ``a photo of a bird'' $\rightarrow$ ``a photo of a bird with a yellow bill.''

Across both shift settings, the adaptive \ensemblearm{} and \llmconcept{} detectors are the first and second fastest, respectively.
For the yellow bill shift, the next fastest methods are adaptive \embeddingarm{} and \saearm{} monitoring and the \informedshift{} fixed monitor. For the white crown shift, all the adaptive monitors achieve similar EDD and substantially outperform the fixed monitors; this suggests this shift was easy to learn across all data representations.

The \llmconcept{} bettor learned the correct concepts across most streams (Fig~\ref{fig:results}f, g).
For the yellow bill shift, the LLM learned concepts that exactly coincided, like ``Does the image show a bird with yellow beak?'', or overlapping concepts, like ``Does the image show a bird with yellow or yellow-orange bill?''
Interestingly, for the white crown shift, the LLM learned the correct concept or its opposite, i.e., ``Does the image show a bird with a dark head or dark crown?''

\subsection{Monitoring AI scribe data}
\label{sec:exp_notes}
\vspace{-0.2cm}

Working with the AI operations team at UCSF, we monitored shifts in real-world drafts of clinical notes generated by AI scribes and corresponding edits by the healthcare provider.
Such shifts can result from changes in the scribe (e.g., model version, prompt) or clinician behavior (e.g., automation bias). Monitoring is needed, as product updates are often unannounced and provide limited insight into how drafts and editing behaviors change.
Here we consider a known model update within the Rheumatology department.
Streams were constructed by drawing on notes $60$ days before and after the update.

Again, we find that all adaptive monitoring procedures were  substantially faster than the fixed monitors (Fig~\ref{fig:results}d).
The ranking between the adaptive monitors differs here: \ensemblearm{} and \embeddingarm{} bettors are now the top-ranking, outperforming \texttt{SAE} and \llmconcept{} alone.
This change is not entirely unexpected, based on the concepts learned
by the \llmconcept{} bettor (Table~\ref{tab:scribe_examples}): 
the model update changed the overall structure of the AI draft and made it more concise, so edits that ``restructure the note'' and ``remove unnecessary, redundant, or irrelevant content'' became less frequent.
Such major changes are more quickly detectable using embeddings, whereas \llmconcept{} is faster at detecting sparse shifts.
Explanations such as these show the utility of the \llmconcept{} bettor.
Indeed, the AI operations team found the interpretable bettor more actionable, as it points to potential directions for future improvements and unanticipated effects of model updates.
These findings also aligned with rheumatologist feedback, as many reported the update improved note structure, reduced unnecessary information, and increased precision of medication information.

\section{Discussion}
\label{sec:discussion}
\vspace{-0.2cm}
One of the most common and basic tasks when deploying a GenAI system is monitoring unstructured data streams.
Whereas existing approaches have relied on detecting shifts in pre-specified directions, this work studies how to adaptively learn what to monitor, given the vast breadth of potential shifts for unstructured data streams.
Drawing on testing-by-betting, this work introduces asymptotic e-detectors, which adaptively bet the shift direction at each time step while providing asymptotic false alarm control.
Across real-world data streams, adaptive bettors are found to detect shifts faster than fixed monitors:
\llmconcept{} bettors were fastest for sparse interpretable shifts, black-box bettors for broad shifts, and their ensemble achieved (near-)top performance across all settings.
Furthermore, by incorporating LLM reasoning, \llmconcept{} bettors can identify shifts of practical relevance, which often accurately coincided with the true shift in experiments.
Finally, while this work shows how reasoning can be incorporated into automated monitoring procedures, it is important to acknowledge that LLMs may identify shifts that are correlated with but are not the actual root cause. Thus expert interpretation of monitoring results is ultimately needed for fully comprehensive monitoring and diagnosis.

\subsection*{AI use statement}

In this work, we used generative AI tools to help with: cleaning and reformatting datasets, providing critical ingredients for proving mathematical claims, assisting in the writing of proofs, designing or providing feedback on research methodology or experiments, and implementing methods.
We have not used generative AI tools to help with: generating synthetic data sets; helping develop theoretical models or conceptual frameworks; formulating mathematical claims; proposing or refining hypotheses; assisting with translation; supporting qualitative and thematic data analysis; and interpreting results.
Additionally, we used generative AI tools to help with: creating figures for illustrating the proposed methods; editing software code; and summarizing existing literature.
All AI-assisted work is carefully reviewed and edited by the authors.
All LLM-generated code was verified and tested for correctness by authors.
All LLM-assisted proofs were checked line by line by the authors.
We take responsibility for the final content of this work, including text, claims, or artifacts produced with the aid of generative AI.

\subsection*{Acknowledgements}

The authors thank Nicholas Petrick, Berkman Sahiner, and Gene Pennello (U.S. Food and Drug Administration, Center for Devices and Radiological Health) for helpful discussions and serving as stakeholders for this project.

\subsection*{Disclaimers}
The contents represent the views of the author(s) and do not represent the views of, and are not an endorsement by, the U.S. Food and Drug Administration (FDA)/Department of Health and Human Services (HHS) or the U.S. Government. The views, findings, and interpretations contained in this document do not constitute FDA guidance, position on this matter, or legally enforceable requirements.

\bibliographystyle{plainnat}
\bibliography{paperpile_llm_changepoint,extra_refs}

\newpage
\section*{Checklist}

\begin{enumerate}

  \item For all models and algorithms presented, check if you include:
  \begin{enumerate}
    \item A clear description of the mathematical setting, assumptions, algorithm, and/or model. [Yes, see Section \ref{sec:framework} and \ref{sec:bettor}]
    \item An analysis of the properties and complexity (time, space, sample size) of any algorithm. [Yes, see Section \ref{sec:framework}]
\item (Optional) Anonymized source code, with specification of all dependencies, including external libraries. [Yes, \url{https://github.com/jjfenglab/monitoring-by-betting}]
  \end{enumerate}

  \item For any theoretical claim, check if you include:
  \begin{enumerate}
    \item Statements of the full set of assumptions of all theoretical results. [Yes]
    \item Complete proofs of all theoretical results. [Yes, see Appendix \ref{app:proofs}]
    \item Clear explanations of any assumptions. [Yes, see Section \ref{sec:framework}]
  \end{enumerate}

  \item For all figures and tables that present empirical results, check if you include:
  \begin{enumerate}
\item The code, data, and instructions needed to reproduce the main experimental results (either in the supplemental material or as a URL). [Yes, \url{https://github.com/jjfenglab/monitoring-by-betting}]
    \item All the training details (e.g., data splits, hyperparameters, how they were chosen). [Yes, see Appendix  \ref{app:implementation}]
    \item A clear definition of the specific measure or statistics and error bars (e.g., with respect to the random seed after running experiments multiple times). [Yes, see Appendix~\ref{app:extra_results} and figures]
    \item A description of the computing infrastructure used. (e.g., type of GPUs, internal cluster, or cloud provider). [Yes, see Appendix~\ref{app:extra_results}]
  \end{enumerate}

  \item If you are using existing assets (e.g., code, data, models) or curating/releasing new assets, check if you include:
  \begin{enumerate}
    \item Citations of the creator If your work uses existing assets. [Yes, see References]
    \item The license information of the assets, if applicable. [Yes, see Appendix~\ref{app:experiments}]
    \item New assets either in the supplemental material or as a URL, if applicable. [Not Applicable]
    \item Information about consent from data providers/curators. [Not Applicable]
    \item Discussion of sensible content if applicable, e.g., personally identifiable information or offensive content. [Yes, see IRB info in Appendix~\ref{app:experiments}]
  \end{enumerate}

  \item If you used crowdsourcing or conducted research with human subjects, check if you include:
  \begin{enumerate}
    \item The full text of instructions given to participants and screenshots. [Not Applicable]
    \item Descriptions of potential participant risks, with links to Institutional Review Board (IRB) approvals if applicable. [Yes, see IRB info in Appendix~\ref{app:experiments}]
    \item The estimated hourly wage paid to participants and the total amount spent on participant compensation. [Not Applicable]
  \end{enumerate}

\end{enumerate}

\onecolumn
\aistatstitle{Supplementary Materials}

\appendix
\renewcommand{\thefigure}{\thesection.\arabic{figure}}
\renewcommand{\thetable}{\thesection.\arabic{table}}
\setcounter{figure}{0}
\setcounter{table}{0}
\counterwithin{figure}{section}
\counterwithin{table}{section}

\section{The Asymptotic E-Detector Framework and Proofs}
\label{app:proofs}

Suppose we observe a stream $X_1, X_2, \dots \in \mathcal{X}$ arriving sequentially and we have a calibration dataset of size $m$, denoted by $\mathcal{C}_m = \{Z_1, \dots, Z_m\}$, drawn i.i.d.\ from the pre-change distribution $P$.
Let $\Pm$ denote the empirical distribution over $\mathcal{C}_m$.
The monitoring problem is to test:
\begin{align*}
    H_0&: \exists \ P, \ \text{s.t.} \ \mathbf{P} = P^{\infty} \\
    H_1&: \exists \tau^*\geq 1 \text{ and } Q \ne P \text{ s.t. } \mathbf{P} = P^{\tau^*} \times Q^{\infty}.
\end{align*}
The operating characteristics of a monitoring procedure are its average run length under the null $\ARL = \mathbb{E}_{H_0}[\tau]$ and its expected detection delay $\EDD = \mathbb{E}_{H_1}[\tau - \tau^* \mid \tau > \tau^*]$, where $\tau$ is the stopping time.
In this section, we use $C$ to denote constants for shorthand.

\subsection{Testing and monitoring by betting}

We begin by reviewing e-processes and e-detectors from the testing-by-betting framework.
More details can be found in \citep{Shin2024-dj} and \citep{Ramdas2025-uv}.

\paragraph{Baseline E-process}
Start with $\$1$ of initial wealth and, at each time $t$, bet against $H_0$ so that the wealth $E_t$ at $t$ is updated by $E_t = E_{t-1} \cdot (1 + \lambda_t r_t)$ where $\lambda_t \in (0,1)$ is a bet scaling factor and $r_t$ is the per-step return.
The bet is \emph{fair} under the null if $\mathbb{E}\bigl[r_t \mid \Fil_{t-1}\bigr] \leq 0$.
Along with the assumption that $1+\lambda_tr_t \geq 0$ for any $t$, $(E_t)_{t \geq 0}$ is a nonnegative supermartingale and thus an e-process under $H_0$.
By Ville's inequality, $\mathbb{P}_{H_0}(\sup_t E_t \geq 1/\alpha) \leq \alpha$ and thus the false alarm rate is controlled at level $\alpha$ uniformly over time.
Under $H_1$, the conditional mean of the return $r_t$ is positive and the wealth $E_t$ grows exponentially to trigger an alarm.

\paragraph{E-detector}
Because $\tau^*$ is unknown, no single e-process suffices: one started before the change expends its capital on pre-change observations.
Following \citet{Shin2024-dj}, at each timepoint $j \ge 1$ a fresh $\$1$ is invested in a new subgame betting against the null starting at index $j$.
The detector is the sum of all subgames:
\begin{equation}
M_t \;=\; \sum_{j \leq t} \Ej_t,
\qquad
\Ej_t \;=\; \prod_{s = j}^{t} (1+\lambda_s\rj_s)
\label{eq:edetector}
\end{equation}
with $M_0 = 0$.
When the subgames are e-processes starting at timepoints $j$, total wealth of all subgames under $H_0$ cannot outrun the money invested: $\mathbb{E}_{H_0}[M_{\tau}] \leq \mathbb{E}_{H_0}[\tau]$ for any stopping time $\tau$.
Defining the stopping time as $\tau = \inf\{t : M_t \geq 1/\alpha\}$ yields $\ARL = \mathbb{E}_{H_0}[\tau] \geq 1/\alpha$.
When $\rj_t$ does not depend on the subgame index $j$, the sum reduces to the Shiryaev--Roberts recursion
\begin{equation}
M_t \;=\; \bigl(1 + M_{t-1}\bigr) \cdot (1+\lambda_t r_t) ,
\label{eq:sr}
\end{equation}
allowing the total wealth $M_t$ to be updated in a single multiplication per step.

\paragraph{Asymptotic E-detectors}
As defined in the main manuscript, a process $\mathbf{M}_m = \{M_{m,t}\}_{t\geq0}$ with an associated stopping time $\tau_m = \inf\{t: M_{m,t}\geq 1/\alpha\} $ is an asymptotic e-detector if $\mathbf{M}_m$ satisfies
\begin{align}
\limsup_{m \to \infty} \left( \mathbb{E}_{H_0}[M_{m,\tau_m}] - \mathbb{E}_{H_0}[\tau_m] \right) \leq 0.
\end{align}

\subsection{Proof of Theorem 1}
Recall the SR detector which reuses the same calibration set $\mathcal{C}_m$ of size $m$ over time defined by
\begin{align*}
    M_{m,t} &= \sum_{j=1}^{t} E_{m,t}^{(j)}, \quad M_{m,0} = 0, \\
    E_{m,t}^{(j)} &= \begin{cases}
      \prod_{s=j}^{t} (1+ \lambda_s r_{m,s}) \ &\text{for }t \geq j \\
      1  &\text{otherwise}
    \end{cases}\\
r_{m,t} &= \hat{f}_t(X_t) - \Pm \hat{f}_t. \label{eq:asymp_Edetector}
\end{align*}
In this subsection, we present the proof of Theorem 1 in the main manuscript that the process $\mathbf{M}_{m} = \{M_{m,t}\}_{t\geq 0}$ is an asymptotic e-detector with asymptotic $\ARL$ control.

Suppose the following standard regularity conditions are satisfied.
\begin{assumption}[Boundedness]
\label{assump:boundedness-measurablility}
The witness functions satisfy $\hat{f}_t(x) \in [0,1]$ for all $t \ge 1$ and all $x \in \mathcal{X}$ and $\hat{f}_t$ is $\mathcal{F}_{t-1}$-measurable.
\end{assumption}
For simplicity, we assume $\hat{f}_t(x) \in [0,1]$ which is equivalent to the assumption in the main manuscript up to an affine transformation.
\begin{assumption}\label{assump:independence}
The calibration dataset $\mathcal{C}_m$ is a random sample from the pre-change distribution $P$ and is independent of the monitoring stream $X_1, X_2, \ldots$
\end{assumption}
\begin{assumption}[Uniformly integrable stopping time]\label{assump:uniformIntegrable-stopping}
The class of stopping time $\{\tau_m\}_{m\in \mathbb{N}} $ is uniformly integrable under $H_0$ (\ref{eq:nullHypothesis}).
More specifically, the following are satisfied by $\{\tau_m\}_m$:
\begin{enumerate}
    \item $\exists$ a finite $T^*$ such that $\forall m\geq 1, \E[\tau_m] \leq T^*$.
    \item $\forall \epsilon, \ \exists \delta$, such that $\forall$ measurable $A$ with $P(A) \leq \delta$ and $\forall m\geq 1$, $\E[|\tau_m|I_A] \leq \epsilon$.
\end{enumerate}
\end{assumption}

Note that we study behavior of the procedure under this uniform integrability assumption because if $\mathbb{E}_{H_0}[\tau_m] \rightarrow \infty$, then $\ARL$ trivially satisfies the run length requirement.

\begin{assumption}[\citep{Kosorok2008-pr}]\label{assump:classifier-GC}
All witness functions $\{\hat{f}_t\}$ lie in a fixed $P$-Glivenko-Cantelli (GC) class $\mathbb{F}$.
\end{assumption}

\begin{proof}
    We decompose the one-step increment of $M_{m,t}$:
    \begin{equation}
    M_{m,t} - M_{m,t-1} = 1 + D_{m,t} + \xi_{m,t},
    \end{equation}
    where $\xi_{m,t}$ is a martingale difference sequence:
    \begin{align}
    \xi_{m,t} &= (M_{m,t} - M_{m,t-1}) - \mathbb{E}_{H_0}[M_{m,t} - M_{m,t-1} \mid \mathcal{F}_{t-1}, \mathcal{C}_m] \nonumber \\
    &= \lambda_t(r_{m,t} - \delta_{m,t})(1 + M_{m,t-1}),
    \end{align}
    where $\delta_{m,t} = \E [r_{m,t} \mid \mathcal{F}_{t-1}, \mathcal{C}_m] = (P - \Pm) \hat{f}_t$ and $D_{m,t}$ is the drift:
    \begin{equation}
    D_{m,t} = \lambda_t \delta_{m,t} (1 + M_{m,t-1}).
    \end{equation}
    Summing to $\tau_m$ and taking expectations:
    \begin{equation}
    \mathbb{E}_{H_0}[M_{m,\tau_m}] - \mathbb{E}_{H_0}[\tau_m] = \mathbb{E}_{H_0}\left[\sum_{t=1}^{\tau_m} D_{m,t}\right] + \mathbb{E}_{H_0}\left[\sum_{t=1}^{\tau_m} \xi_{m,t}\right].
    \label{eq:decompose}
    \end{equation}
    It is sufficient to show that the right hand side is upper bounded by a term converging to 0.

    \textbf{Martingale term:}
    The partial sum process $S_T := \sum_{t=1}^{T} \xi_{m,t}$ is a martingale with respect to $\sigma(\Fil_T, \mathcal{C}_m)$.
    For $t \le \tau_m$, $M_{m,t-1} < 1/\alpha$ and, by Assumption \ref{assump:boundedness-measurablility}, $|r_{m,t} - \delta_{m,t}| \le 2$, so increments are uniformly bounded: $|\xi_{m,t}| \le 2 (1 + 1/\alpha)$.
    By Assumption \ref{assump:uniformIntegrable-stopping}, there exists $T^*$ such that $\mathbb{E}_{H_0}[\tau_m] \le T^* < \infty$.
    Then, the Optional Stopping Theorem implies
    \begin{equation*}
    \mathbb{E}_{H_0}\left[\sum_{t=1}^{\tau_m} \xi_{m,t}\right] = 0.
    \end{equation*}

    \textbf{Drift term:}
    Since $M_{m,t-1} < 1/\alpha$ for all $t \le \tau_m$ and $\lambda_t \in (0,1)$, we have
    \begin{align*}
         \left| \sum_{t=1}^{\tau_m} D_{m,t} \right|
            \leq C \sum_{t=1}^{\tau_m} |\delta_{m,t}|
            \leq C \tau_m \sup_t |(P - \Pm)\hat{f}_t| := \Delta_m
    \end{align*}
    for some positive constant $C$.
    Since $\hat{f}_t$ is bounded, $\sup_t |(P - \Pm)\hat{f}_t|$ is positive and bounded.
    Together with uniform integrability of $\tau_m$, $\Delta_m$ is also uniformly integrable.
    Moreover, $\Delta_m \to 0$ in probability since all witness functions lie in a Glivenko-Cantelli class and $\tau_m$ is uniformly integrable.
    Therefore, by Vitali convergence theorem \citep{Kallenberg2021-tg},
    \begin{align*}
        \E \left| \sum_{t=1}^{\tau_m} D_{m,t} \right| \leq \E[\Delta_m] \to 0,
    \end{align*}
    establishing the claim.
\end{proof}

\begin{corollary}[Asymptotic $\ARL$ control]
Asymptotic e-detectors provide asymptotic $\ARL$ control under the null:
\begin{equation}
\liminf_{m \to \infty} \mathrm{ARL}_{0,m} = \liminf_{m \to \infty} \mathbb{E}_{H_0}[\tau_m] \geq \frac{1}{\alpha}.
\end{equation}
\end{corollary}

\begin{proof}
By definition of $\tau_m = \inf\{t : M_{m,t} \geq 1/\alpha\}$, $M_{m,\tau_m} \geq 1/\alpha$ whenever $\tau_m < \infty$.
Hence $\frac{1}{\alpha} \leq \mathbb{E}_{H_0}[M_{m,\tau_m}]$.
Since $M_m$ is an asymptotic e-detector, for any $\epsilon > 0$ there exists $m_0$ such that for all $m \geq m_0$:
\begin{equation}
\frac{1}{\alpha} \leq \mathbb{E}_{H_0}[M_{m,\tau_m}] \leq \mathbb{E}_{H_0}[\tau_m] + \epsilon.
\end{equation}
Rearranging gives $\mathbb{E}_{H_0}[\tau_m] \geq 1/\alpha - \epsilon$ for all $m \geq m_0$.
Since $\epsilon > 0$ was arbitrary, taking the limit inferior yields $\liminf_{m \to \infty} \mathbb{E}_{H_0}[\tau_m] \geq 1/\alpha$.
\end{proof}

\section{Implementation Details}
\label{app:implementation}

\subsection{Adaptive betting hyperparameters}

\paragraph{Bet scaling factor.}
The bet scaling factor can adapt over time and should balance between aggressive betting (high $\lambda$, faster detection but higher variance) and conservative betting (low $\lambda$, slower but more stable).
While adaptive tuning of $\lambda_t$ is possible within the e-detector framework and may improve performance, this work uses a fixed value of $\lambda =0.5$ and leaves further optimization to future work.

\paragraph{Changepoint learner.}
The changepoint learner estimates when the distribution shifted, which determines the pseudo-labels for classifier training.
The mean-shift estimator (Equation~\ref{eqn:cp_meanShift}) is appropriate when shifts are expected to be dense across features, as in embedding spaces where many coordinates may move together.
The sparse estimator (Equation~\ref{eqn:cp_sparse}) is preferable when shifts are expected to affect only a subset of features, as with interpretable concepts where a small number of attributes may change.
In our experiments, we use the mean-shift estimator for the black-box bettor (operating on dense embeddings) and the sparse estimator with maximum subset size $B = 4$ for the interpretable bettor (operating on disentangled concepts).

\paragraph{Classifier.}
The classifier maps features to a score distinguishing pre- from post-change observations.
Any classifier producing calibrated probabilities is suitable; we use logistic regression for its simplicity and interpretability, and apply regularization to prevent overfitting when the data representation is high-dimensional.

\paragraph{Calibration dataset size.}
The calibration dataset size $m$ controls the accuracy of the empirical approximation to the pre-change distribution.
Larger $m$ reduces calibration error and tightens $\ARL$ control, but comes with practical costs: more pre-change data must be available before monitoring begins, and for interpretable bettors, the calibration set must be annotated with each proposed concept, incurring LLM query costs linear in $m$.
Fortunately, our experiments found $\ARL$ control to be robust across a wide range of calibration sizes; see Appendix~\ref{app:vary_m} for simulations demonstrating that even $m = 30$ can provide effective control.

\paragraph{Burn-in period.}
An initial burn-in period is needed so that there is sufficient sample size to meaningfully estimate the changepoint and train the classifier.
During burn-in, the e-detector does not place bets, so $M_t$ remains at its initial value of 0.
The burn-in period should be set such that data during burn-in is nearly guaranteed to be pre-change, to avoid significant delays in detecting shifts.

\subsection{The LLM-concept bettor}
\label{app:llm_bertopic}

The \llmconcept{} bettor uses a four-stage pipeline that interleaves LLMs and unsupervised clustering of text to propose and select concepts at each time point. This four-step procedure is used instead of a zero-shot LLM pipeline, as the latter is often impractical and far from optimal for monitoring long data streams: an LLM is unlikely to be able to fit the entire history into its context window and the LLM is likely to forget observations in the middle of the stream.
Instead, this pipeline combines LLM reasoning with statistical/data-analytic tools, which can be viewed as an agentic process.
This leads to a pipeline that essentially conducts two rounds of changepoint detection: steps 1-2 are to brainstorm topics and changepoints, while steps 3-4 are to formalize the concepts, annotations, and changepoint estimation.
This pipeline is modeled after the agentic pipeline in \citep{Feng2025-qd}, which was originally designed for classification and is extended here to the task of changepoint detection.

\paragraph{Step 1: Keyphrase extraction.}
For each past observation, an LLM is asked to extract a set of keyphrases describing its content according to some \texttt{keyphrase-extraction-prompt}.
(For image streams, we use a VLLM.)
The \texttt{keyphrase-extraction-prompt} provides the monitoring team a way to communicate their priorities regarding which general categories of concepts and shifts are of practical interest.
Figure~\ref{fig:prompt_keyphrase} shows an example prompt for the CEBaB experiment.
Note that keyphrase extraction should be viewed as a brainstorming step, as it is not formal concept extraction nor concept selection.
Rather a shift in the types of keyphrases can provide clues as to the types of shifts that are happening in the data.

\begin{figure}[H]
\centering
\begin{tcolorbox}[colback=gray!5, colframe=gray!50
]
\small\ttfamily
Extract all concepts mentioned this restaurant review. These will be candidates for a central dictionary of concepts for characterizing restaurant reviews, and will eventually be input features to a concept bottleneck model for monitoring changes in reviews written over time. Only list factors explicitly mentioned in the review. Do not infer or add details unless explicitly stated.

For each concept mentioned in the review, provide:

1. A word or phrase to represent that concept (1-2 words)

2. Synonyms or alternative phrasings

3. Broader generalizations that capture similar concepts

Requirements:

- Each descriptor should be <=2 words. One word descriptors are encouraged.

- The descriptor should state distinguishing features. Focus on factors relevant to the dining experience.

- List specific terms, synonyms, and increasingly general concepts.

- Separate items with commas\\[0.5em]

Review(s): "\{note\}"\\[0.5em]
Output Format:\\
\{ "reasoning": "...", "keyphrases": [...] \}
\end{tcolorbox}
\vspace{-0.5cm}
\caption{Keyphrase extraction prompt (abbreviated) for CEBaB experiment.}
\label{fig:prompt_keyphrase}
\end{figure}

\paragraph{Step 2: Topic modeling and initial changepoint estimation.}
Keyphrases are clustered into topics using BERTopic \citep{Grootendorst2022-jl}.
To speed up BERTopic, one can further use an online variant, as long as each of its internal components allow for online learning.
For instance, our experiments use BERTopic with incremental PCA and clusters keyphrases into topics with mini-batch $k$-means.
Then based on the extracted keyphrases, each observation is mapped to all its cluster assignments (binary vector with 1s for clusters that are present and 0s for clusters that are absent).
We then estimate a working changepoint $\hat{\tau}^{(topic)}$ using the sparse changepoint statistic (Equation~\ref{eqn:cp_sparse}) on this topic representation, which identifies both the candidate changepoint and the topics with the largest prevalence shifts.

\paragraph{Step 3: Concept proposal.}
An LLM is prompted to propose formal yes/no concept questions based on the top-shifted BERTopic topics and their representative documents.
For each of the top topics, the LLM prompt includes each topic's  keywords, representative documents, and how the prevalence of the topic has shifted.
Figure~\ref{fig:prompt_proposer} shows one such example for CEBaB.

\begin{figure}[H]
\centering
\begin{tcolorbox}[colback=gray!5, colframe=gray!50
]
\small\ttfamily
We are monitoring for concepts that shift in prevalence in restaurant reviews that are posted over time. A concept has to be defined in terms of a yes/no question of the form "Does the review mention the...?". [...] We have 3 concepts so far:\\
* X0 = Does the review mention the ambience or atmosphere of the restaurant, including aspects like crowding, romantic setting, or a pleasant/charming atmosphere?\\
* X1 = Does the review mention the food being well-prepared or the variety of food offered?\\
* X2 = Does the review mention the ambiance or atmosphere of the restaurant?\\[0.5em]
To come up with the 4th concept, we conducted a data-driven analysis using topic modeling (BERTopic) on extracted keyphrases. Below are the topics that show the largest shifts in prevalence over time (excluding topics already used for existing concepts). [...]\\[0.5em]
* topic 22 (shift: -0.21): "experience" (n=7), "staff performance" (n=4), "friendly staff" (n=3), "crowded" (n=3), "overall experience" (n=3), "positive experience" (n=2), "good experience" (n=2)\\
* topic 9 (shift: +0.18): "atmosphere" (n=33)\\
* topic 11 (shift: +0.15): "dining experience" (n=22), "dining environment" (n=2), "dining occasion" (n=1), "dining time" (n=1), "dining purpose" (n=1), "dining service" (n=1), "restaurant environment" (n=1)\\
* topic 0 (shift: +0.15): "price" (n=13), "expensive" (n=8), "cost" (n=6), "overpriced" (n=4), "high cost" (n=4), "highly rated" (n=2), "pricey" (n=2)\\[0.5em]
[...]\\
YOUR TASK: Propose 2 candidate concepts to use as inputs into a final interpretable changepoint estimation method, based on the topic analysis. [...]\\[0.5em]
Output Format:\\
\{ "reasoning": "...", "concepts": [\{"definition": "...", "topic\_indices": [...]\}, ...] \}
\end{tcolorbox}
\vspace{-0.5cm}
\caption{Concept proposal prompt (abbreviated) for CEBaB experiment. Top keyphrase clusters (topics) are shown from an example replicate.}
\label{fig:prompt_proposer}
\end{figure}

\paragraph{Step 4: Concept annotation and selection.}
For each of the proposed concepts, the LLM is asked to extract the answer to the formal yes/no concept question for each observation in the history.
Applying sparse changepoint estimation (Equation~\ref{eqn:cp_sparse}) to the LLM-extracted concepts, one finally obtains the final changepoint estimate $\hat{\tau}_t$ for formalized concepts, as well as the selection of top concepts.
These selected concepts, coupled with the LLM, then define the data representation function $\phi_{\texttt{C}_t}$ used for betting.
Figure~\ref{fig:prompt_annotator} shows an example annotation prompt.

\begin{figure}[H]
\centering
\begin{tcolorbox}[
colback=gray!5,
colframe=gray!50
]
\small\ttfamily
You will be given a restaurant review. Answer the following question with a probability from 0 to 1. Answer with 0 or 1 when possible, but a probability between 0 and 1 is allowed if needed to express uncertainty.\\[0.5em]
Question: \{prompt\_question\}\\[0.5em]
Restaurant review: \{text\}\\[0.5em]
Example answer:\\
\{ "reasoning": "...", "answer": 1 \}
\end{tcolorbox}
\vspace{-0.5cm}
\caption{Concept annotation prompt for the CEBaB experiment.}
\label{fig:prompt_annotator}
\end{figure}

Finally, we note that the concept extractor must not itself exhibit temporal shifts, as this would violate $\ARL$ guarantees.
In our implementation of the LLM concept extractor, we ensure stationarity by processing each observation independently, extracting each concept independently, and retrying failed LLM queries rather than imputing default values (as query failure rates may also vary over time).
Future work may consider how to ensure temporal consistency of LLM concept extraction over longer periods of time, where LLM providers may silently switch model versions or hardware.

\section{Experiments}
\label{app:experiments}
\subsection{Experimental settings}

Table~\ref{tab:dataset_summary} summarizes the experimental setup for each dataset.
All experiments use 100 simulation replicates for monitors other than \llmconcept{} and \ensemblearm{} and 30 replicates for \llmconcept{} and \ensemblearm{} due to cost constraints.
On CEBaB and CUB, the non-LLM bettors also run on $100$ null streams to illustrate $\ARL$ control.
The code for the \llmconcept{} bettor is carefully implemented such that control of the $\ARL$ for non-LLM bettors also implies $\ARL$ control for the \llmconcept{} bettor.
In particular, the LLM only affects the bettor through the choice of features, which is made from past observations, so the witness at step $t$ is still fit before $X_t$ is observed and never uses the calibration set, as Conditions~(i) and~(iii) of Assumption~\ref{ass:regularity} require.
Each annotation call labels a single observation for a single concept, so $X_t$ and the calibration observations are annotated by the same procedure, independently of one another.
Other than the choice of features, the \llmconcept{} bettor uses the same classifier, plug-in centering on the calibration set, and SR recursion as the \embeddingarm{} and \saearm{} bettors, whose simultaneous $95\%$ bands for the empirical $\ARL$ on the null streams contain or lie above the nominal $\ARL$ (Figures~\ref{fig:cebab_calibration} and~\ref{fig:cub_calibration}).

\begin{table*}[t]
  \centering
  \caption{Summary of experimental settings across the three datasets.}
  \label{tab:dataset_summary}
  \small
  \begin{tabular}{lp{3cm}p{3cm}p{3cm}p{3cm}}
    \toprule
    & \textbf{CEBaB} & \textbf{CUB: yellow bill} & \textbf{CUB: white crown} & \textbf{AI scribe} \\
    \midrule
    Data & Restaurant reviews from CEBaB & Bird images from CUB-200-2011 & Bird images from CUB-200-2011 & HPI section of Rheumatology clinical notes \\
    \midrule
    Shifted attribute & Whether review mentions ambiance & Whether bird has yellow bill in image & Whether bird has white crown in image & Clinical notes from 60 days before vs after model update on 2026-02-25 \\
    Prevalence pre$\to$post & 0.52$\to$0.88 & 0.17$\to$0.64 & 0.30$\to$0.95 & N/A \\
    \midrule
    Embedding model & MiniLM & CLIP ViT-B/32 & CLIP ViT-B/32 & MiniLM \\
    Embedding dim $d$ & 384 & 512 & 512 & 384 \\
    \midrule
    $n_{\text{pre}}$ & 50 & 50 & 50 & 100 \\
    $n_{\text{post}}$ & 300 & 300 & 300 & 150 \\
Burn-in & 30 & 30 & 30 & 75 \\
    \bottomrule
  \end{tabular}
\end{table*}

\begin{table*}[t]
  \centering
\caption{Text exemplars used to define pre-specified drift directions for fixed monitors. Each direction is defined by an in-class (IC) and out-of-class (OC) text description, embedded to form a unit drift vector.}
  \label{tab:ctm_exemplars}
  \small
  \begin{tabular}{l l p{5.5cm} p{5.5cm}}
    \toprule
    \textbf{Dataset} & \textbf{Direction} & \textbf{IC text} & \textbf{OC text} \\
    \midrule
    \multirow{2}{*}{CEBaB}
      & Informed & a typical restaurant review describing the customer's experience of the visit & a restaurant review that comments on the ambiance or decor of the restaurant \\
      & Unrelated & a typical restaurant review describing the customer's experience of the visit & a restaurant review that mentions parking the car near the restaurant \\
    \midrule
    \multirow{2}{*}{CUB (yellow bill)}
      & Informed & a photo of a bird & a photo of a bird with a yellow bill \\
      & Unrelated & a photo of a bird & a photo of a bird perched on a wire or fence post \\
    \midrule
    \multirow{2}{*}{CUB (white crown)}
      & Informed & a photo of a bird & a photo of a bird with a white crown \\
      & Unrelated & a photo of a bird & a photo of a bird perched on a wire or fence post \\
    \midrule
    \multirow{2}{*}{AI scribe}
& Informed & an initial draft from an AI scribe with major editing by the clinician, deleting sections and removing verbose and irrelevant content & an initial draft from an AI scribe with major editing by the clinician, with minimal to light editing like medication name fixes \\
      & Unrelated & an initial draft from an AI scribe with major editing by the clinician & an initial draft from an AI scribe with major editing by the clinician where draft becomes less accurate and needs more editing \\
    \bottomrule
  \end{tabular}
\end{table*}

\paragraph{Pool construction.}
For CEBaB and CUB, we construct disjoint pre-change (p0) and post-change (p1) pools and inject a controlled shift by upsampling observations with a target attribute in the post-change pool.
Specifically, p1 samples are drawn with probability proportional to $\sigma(\beta_0 + \beta \cdot a(x))$, where $a(x) \in \{0,1\}$ is the binary attribute indicator and $\sigma$ is the logistic function.
We use $\beta = 4$, $\beta_0 = -4$ for CEBaB and $\beta = 6$, $\beta_0 = -4$ for CUB.
For CEBaB, we first sample one review variant per original review from the CEBaB counterfactual edits, ensuring each unique review appears at most once across all pools.
For CUB, the population is restricted to seabird genera (Gull, Tern, Albatross, Jaeger, Puffin, Cormorant, Fulmar, Kittiwake, Auklet, Merganser, and Pelican), comprising 1,945 images across 34 species.
For the white crown attribute, we use consensus labels that require an image's annotation to agree with the majority label of its species, as raw CUB annotations exhibit within-species variation.
In every experiment, a stream draws its calibration set and pre-change observations from the pre-change pool without replacement and its post-change observations from the post-change pool, and null streams draw all observations from the pre-change pool.
For CEBaB, the reviews are split by original review into three disjoint sides, with $20\%$ for training the \texttt{SAE}, $55\%$ for the pre-change pool and $25\%$ for the post-change pool, and the pre-change pool and the \texttt{SAE} side keep each variant with equal probability.
For CUB, the images are randomly permuted and the first $730$ form the pre-change pool, the post-change pool is $300$ images drawn from the remaining $1{,}215$, and the $915$ images in neither pool train the \texttt{SAE}.

\paragraph{AI scribe.}
The AI scribe dataset consists of History of Present Illness (HPI) sections from rheumatology encounters, collected under IRB approval at UCSF.
Each observation is the structured difference between the AI-generated draft and the clinician-edited final note.
The shift corresponds to a vendor model update intended to improve performance for rheumatology; the vendor did not disclose specifics of the update beyond its date, though clinicians reported subjective improvements in draft quality.
Because this is a real-world temporal shift without ground-truth labels, prevalence statistics are not available.
The pre- and post-change pools hold the HPI sections from the $60$ days before and the $60$ days after the update, $300$ and $190$ notes, respectively.
Each edit is a word-level deletion, insertion or replacement between the draft and the signed note, shown with three words of surrounding context.
The \texttt{SAE} is trained on a random $20\%$ of the pre-change notes, which streams never draw.

\paragraph{Adaptive bettor settings.}
Every adaptive bettor bets $1 + \lambda (\hat{f}_t - \hat{f}_{\text{null}})$ with $\lambda = 0.5$ against a calibration set of $m = 30$ pre-change observations.
The changepoint statistics use whitened features and a minimum block length of $10$, and the \saearm{} bettor, like the \llmconcept{} bettor, selects a subset of its latent features with the sparse statistic.
The classifier is a logistic regression on standardized features with balanced class weights,
fit with a ridge penalty for the \embeddingarm{} bettor and no penalty otherwise.
The \saearm{} is a TopK sparse autoencoder with $32$ latent features and $K = 3$.
The \ensemblearm{} averages the e-detectors $M_t$ of the \llmconcept{} and \embeddingarm{} bettors on the same stream, as the average of e-detectors is itself an e-detector.
BERTopic for the \llmconcept{} bettor clusters keyphrases into $25$ topics at every update.

\paragraph{Baselines.}
The fixed monitors run the Shiryaev--Roberts procedure on conformal p-values with the Simple Jumper betting strategy (jump rate $0.01$) \citep{Vovk2021-re}.
Half of the calibration set serves as the reference set for the nonconformity score and the other half seeds the conformal p-values.
\centroidarm{} scores an observation by its cosine distance to the centroid of the reference set, and the guided variants score it by its cosine similarity with a fixed direction.
For \informedshift{} and \unrelatedshift{}, the direction is the normalized difference between the embeddings of the two descriptions in Table~\ref{tab:ctm_exemplars}, using the text encoder that matches the observation embedding.
For \oraclearm{}, the direction is the normalized difference between the mean embeddings of the post- and pre-change pools.

\subsection{Additional results}
\label{app:extra_results}

\paragraph{\texorpdfstring{ARL$_0$}{ARL0} is controlled.}
Figures~\ref{fig:cebab_calibration} and~\ref{fig:cub_calibration} show the empirical $\ARL$ on the CEBaB and CUB null streams.

\paragraph{Learned concepts for AI scribe experiment.}
Table~\ref{tab:scribe_themes} groups the AI scribe concepts of Figure~\ref{fig:results}(h) into themes, and Table~\ref{tab:scribe_examples} lists the concepts of four AI scribe streams.

\paragraph{Simultaneous confidence bands.}
We show each performance curve with a bootstrapped $95\%$ simultaneous confidence band.

\paragraph{Compute.}
All experiments ran on CPUs on an internal academic Slurm cluster.
Each monitoring stream ran as its own single-node job with $2$ CPU cores and $10$ GB of memory, and the sentence embeddings and SAE training also ran on CPU.
LLM calls went to \texttt{gpt-4.1-mini} through an institution-approved Azure OpenAI deployment.

\paragraph{Licenses.}
CEBaB is distributed under CC BY 4.0, and CUB-200-2011 is restricted to non-commercial research and educational use.
The MiniLM sentence encoder, the sentence-transformers library, and HypotheSAEs are released under Apache-2.0, and BERTopic and the CLIP code are released under MIT.
We accessed \texttt{gpt-4.1-mini} under the Azure OpenAI service terms, and the AI scribe data were used under IRB approval and are not publicly released.

\paragraph{Cost.}
At \texttt{gpt-4.1-mini} prices of \$0.40 and \$1.60 per million input and output tokens, the median \llmconcept{} stream cost \$2.26 on CEBaB, \$3.42 on CUB yellow bill, \$1.64 on CUB white crown and \$13.59 on AI scribe.
A stream took a median of $27$ minutes, ranging from $11$ minutes to $1.8$ hours.

\begin{figure}[t]
  \centering
  \includegraphics[width=0.42\linewidth]{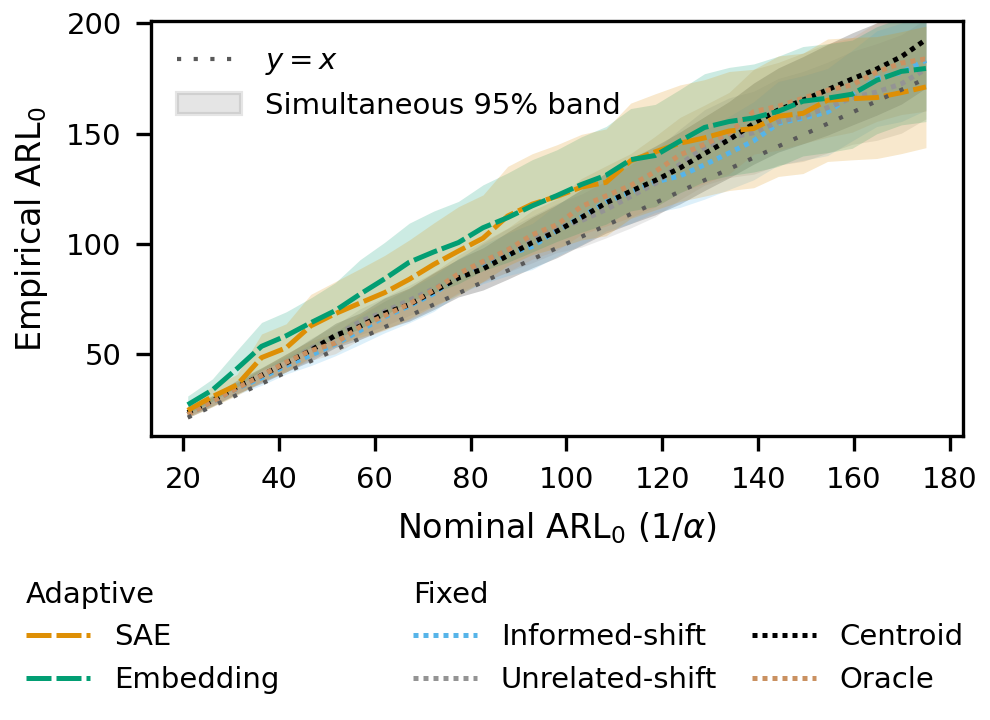}
  \caption{Empirical $\ARL$ against nominal $\ARL$ on the $100$ CEBaB ambiance null streams.
  Shaded regions are simultaneous $95\%$ bands over streams, and the dotted line marks empirical $\ARL$ equal to nominal $\ARL$.}
  \label{fig:cebab_calibration}
\end{figure}

\begin{figure*}
  \centering
  \includegraphics[width=0.42\textwidth]{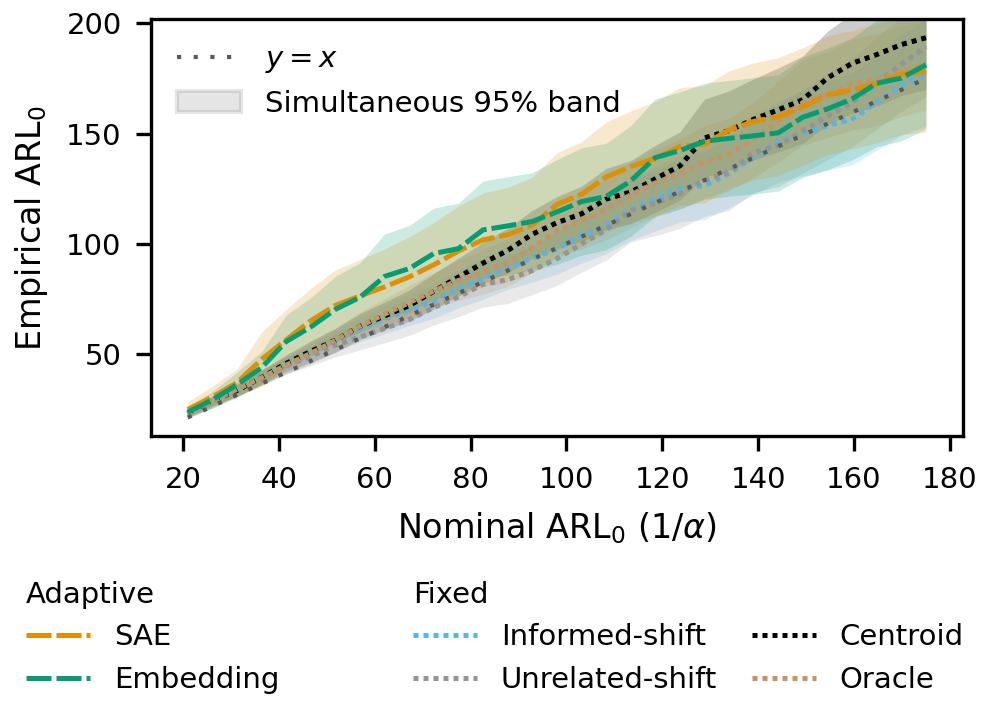}
  \hfill
  \includegraphics[width=0.42\textwidth]{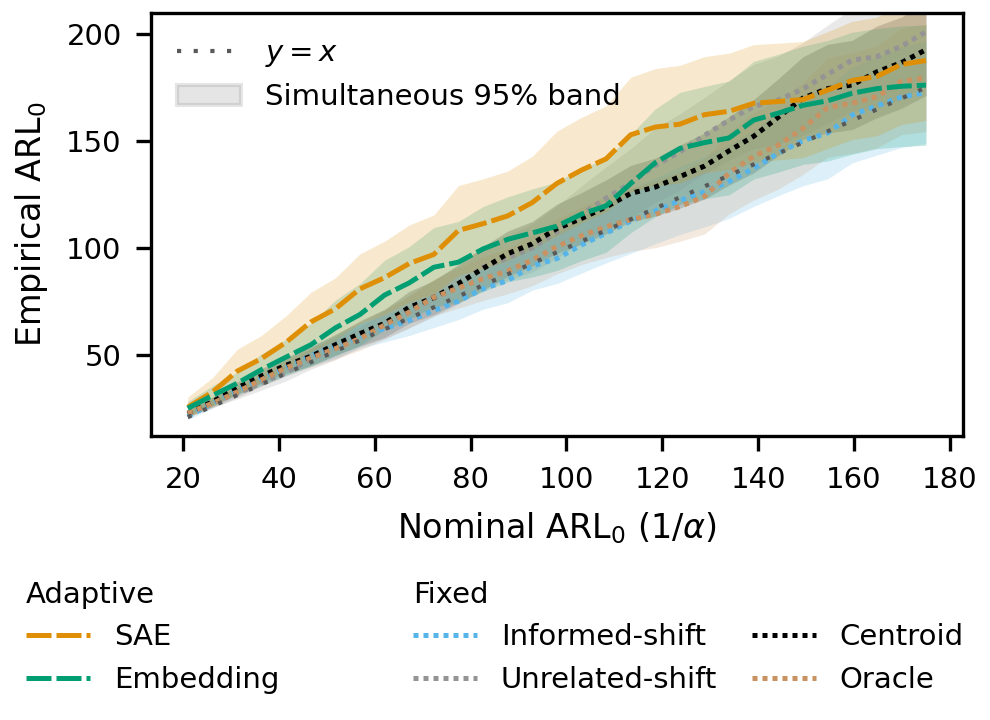}
  \caption{Empirical $\ARL$ against nominal $\ARL$ on the CUB yellow bill (left) and white crown (right) null streams, over $100$ null streams per setting.
  Shaded regions are simultaneous $95\%$ bands over streams, and the dotted line marks empirical $\ARL$ equal to nominal $\ARL$.}
  \label{fig:cub_calibration}
\end{figure*}

\begin{table*}[t]
  \centering
  \caption{Themes of the AI scribe concepts in Figure~\ref{fig:results}(h), over the four concepts active at the end of each of $18$ streams.
  An LLM judge (\texttt{gpt-4.1-mini}) assigns each concept to a theme from its text alone, using a fixed codebook based on the vendor's stated purpose for the update and the clinicians' reports of editing less.
  The first two themes come from these reports; the third theme is named after reading concepts that fit neither.
  Before $\to$ after shows the mean fraction of notes for which a concept is annotated as present, before and after the changepoint.
  }
  \label{tab:scribe_themes}
  \footnotesize
  \setlength{\tabcolsep}{4pt}
\begin{tabular}{p{1.6cm}p{3.8cm}ccp{5.2cm}}
    \toprule
    Theme & Codebook definition & \# Concepts & Before $\to$ after & Example concepts (before $\to$ after) \\
    \midrule
    Unneeded content & Draft content that does not belong in the note because it is unnecessary, redundant, verbose, non-clinical, introductory or identifying, or the removal, condensing or restructuring of such content, whether clinical or not & 46 & 0.79 $\to$ 0.60 & Does the initial draft include unnecessary or irrelevant non-clinical information such as review of systems, introductory sentences, or personal details? ($0.93 \to 0.60$) \newline Does the edit remove redundant or unnecessary note sections (including review of systems or initial summary)? ($0.86 \to 0.54$) \\
    \midrule
    Clinical correction & Clinical content (diagnoses, symptoms, medications, treatment) that the clinician adds, corrects or clarifies, or that the draft is missing & 21 & 0.56 $\to$ 0.61 & Does the edit correct or improve diagnostic information? ($0.38 \to 0.59$) \newline Does the initial draft miss or have incorrect medication, treatment, or procedure details? ($0.59 \to 0.43$) \\
    \midrule
    Non-clinical correction & Non-clinical information (patient or social history, follow-up, administrative or demographic details) that the clinician adds, corrects or clarifies, or that the draft is missing & 5 & 0.51 $\to$ 0.67 & Does the edit improve the accuracy and completeness of patient, social history, or provider demographic information? ($0.50 \to 0.68$) \\
    \bottomrule
  \end{tabular}
\end{table*}

\begin{table*}[t]
  \centering
  \caption{Concepts active at the end of the first four AI scribe streams in Figure~\ref{fig:results}(h), with the theme each is assigned in Table~\ref{tab:scribe_themes} and the fraction of notes for which it is annotated as present before and after the update.}
  \label{tab:scribe_examples}
  \footnotesize
  \setlength{\tabcolsep}{4pt}
  \begin{tabular}{c p{9.6cm} l c}
    \toprule
    Stream & Concept & Theme & Before $\to$ after \\
    \midrule
    \multirow{4}{*}{1} & Does the edit add or correct procedural or diagnostic details based on inaccuracies or omissions in the initial draft? & Clinical correction & $0.56 \to 0.67$ \\
     & Does the edit clarify or correct medication, management, or demographic information based on inaccuracies or omissions in the initial draft? & Clinical correction & $0.61 \to 0.72$ \\
     & Does the edit add or correct missing or inaccurate patient history and administrative details (such as functional status, anatomical location, provider identification, or lab data) in the initial draft? & Non-clinical correction & $0.63 \to 0.76$ \\
     & Does the edit add or correct missing treatment and clinical history details based on inaccuracies or omissions in the initial draft? & Clinical correction & $0.53 \to 0.64$ \\
    \midrule
    \multirow{4}{*}{2} & Does the edit remove unnecessary or irrelevant patient-related information (such as identification, introduction, background, history, or personal details) from the initial draft? & Unneeded content & $0.82 \to 0.66$ \\
     & Does the edit improve procedural and clinical detail documentation by clarifying, correcting, or adding missing information based on the initial draft? & Clinical correction & $0.53 \to 0.72$ \\
     & Does the edit improve the note content by adding, clarifying, condensing, or editing the clinical note content? & Clinical correction & $0.64 \to 0.84$ \\
     & Does the edit remove patient identifiers or identifiable patient details from the initial draft? & Unneeded content & $0.78 \to 0.52$ \\
    \midrule
    \multirow{4}{*}{3} & Does the initial draft include unnecessary or irrelevant non-clinical information such as review of systems, introductory sentences, or personal details? & Unneeded content & $0.93 \to 0.60$ \\
     & Does the edit remove patient identifiers or sensitive non-clinical information? & Unneeded content & $0.81 \to 0.54$ \\
     & Does the edit remove unnecessary, redundant, or irrelevant content? & Unneeded content & $0.91 \to 0.68$ \\
     & Does the edit restructure, clarify, or remove unnecessary note structure or introductory content? & Unneeded content & $0.95 \to 0.73$ \\
    \midrule
    \multirow{4}{*}{4} & Does the edit restructure the note and remove unnecessary introductory sentences or note sections from the initial draft? & Unneeded content & $0.85 \to 0.61$ \\
     & Does the edit correct inaccurate information in the initial draft, such as medication orders, clinician names, or treatment details? & Clinical correction & $0.54 \to 0.74$ \\
     & Does the edit remove personal or demographic sensitive information from the initial draft? & Unneeded content & $0.79 \to 0.54$ \\
     & Does the edit remove patient identifiers or personal information from the initial draft? & Unneeded content & $0.76 \to 0.58$ \\
    \bottomrule
  \end{tabular}
\end{table*}

\section{Additional Experiments}
\subsection{Evaluating \texorpdfstring{ARL$_0$}{ARL0} control for varying calibration dataset sizes}
\label{app:vary_m}

To evaluate $\ARL$ control for finite calibration dataset sizes $m$, we simulate null streams at varying values of $m$ for two settings: a fixed versus learned witness function.
In the fixed-score setting, scores $\hat{f}_t(X_t)$ and $\hat{f}_t(Z_i)$ are drawn i.i.d.\ from a $\mathrm{Beta}(2,2)$ distribution.
In the learned-score setting, the data $X_t$ is drawn from a $5$-dimensional standard normal distribution with the bettor continually refitting the changepoint and retraining a logistic classifier.
The calibration dataset size $m$ is varied across values $\{10, 30, 100, 1000\}$.

As shown in Figure~\ref{fig:varym_fitted_vs_score}, the empirical $\ARL$ is larger than the nominal $\ARL$ in nearly every setting except for the smallest calibration dataset size of $m=10$ and the largest target $\ARL$ value of $10000$.
Interestingly, the empirical $\ARL$ is extremely conservative for the setting with fixed witness functions, while the empirical $\ARL$ is much closer to the nominal $\ARL$ value for the learned witness.
In addition, Figure~\ref{fig:varym_alarm_time_cdf} compares the distribution of null alarm times at $1/\alpha=300$, within the range of thresholds in our experiments, with that of the oracle detector, which uses the exact pre-change expectation.
With $m=30$, the distribution is similar to the oracle's for both the fixed and the learned score, and smaller calibration datasets move some alarms earlier.
Based on these results, our experiments use $m=30$.

\begin{figure}[tbp]
\centering
\includegraphics[width=\linewidth]{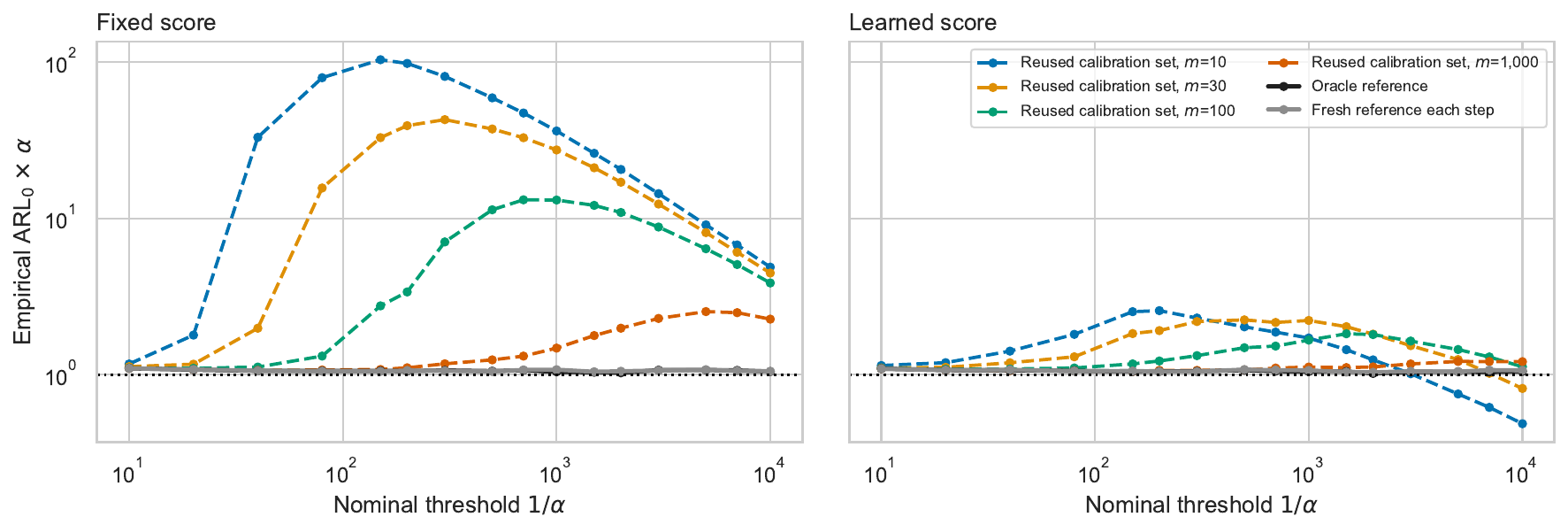}
\caption{Empirical $\ARL$ as a multiple of the nominal $\ARL = 1/\alpha$ with a fixed score (left, $\mathrm{Beta}(2,2)$ scores) and with a learned score from a classifier refit on the stream history (right).
Dashed colored lines reuse a calibration dataset of size $m \in \{10, 30, 100, 1000\}$ at every step, the solid black line uses the oracle expectation $\mathbb{E}_{P_0}[\hat{f}_t(X)]$, and the solid grey line redraws $10$ reference observations at every step.
The dotted line marks the nominal $\ARL$.
Each point summarizes $2{,}000$ null streams, and streams without an alarm after $100{,}000$ steps are counted at $100{,}000$, so each value is a lower bound.}
\label{fig:varym_fitted_vs_score}
\end{figure}

\begin{figure}[H]
\centering
\includegraphics[width=0.8\linewidth]{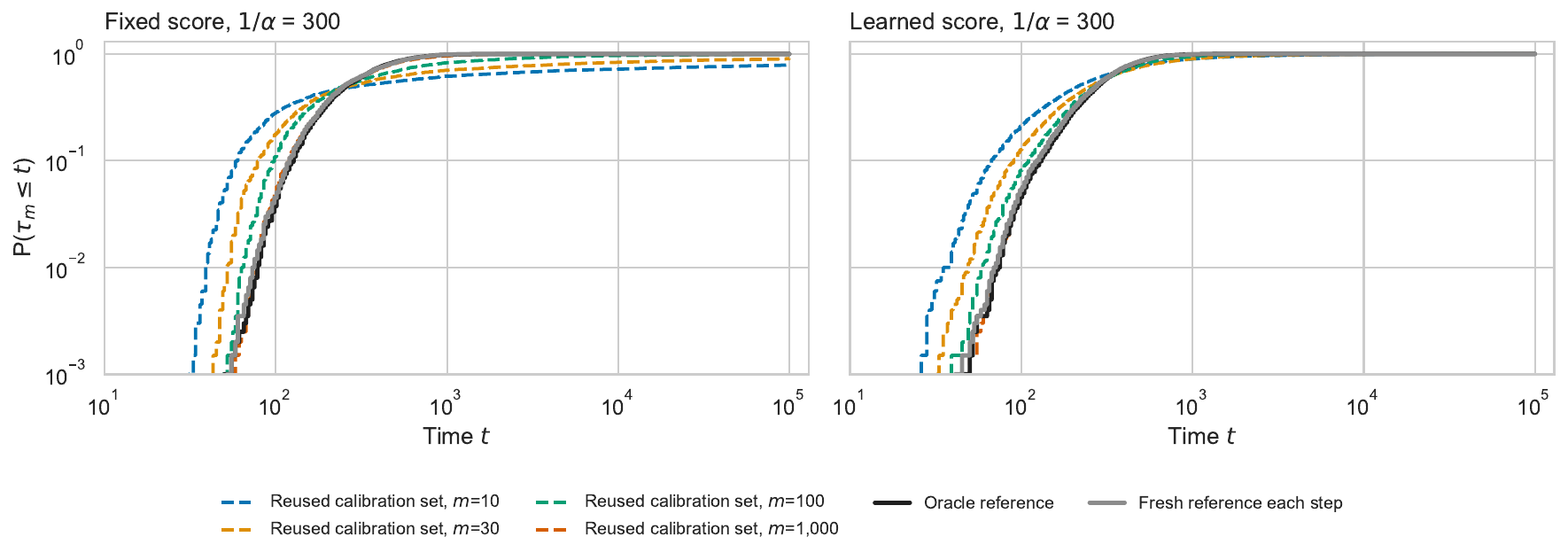}
\caption{Distribution of the null alarm time $\tau_m$ at $1/\alpha = 300$ with a fixed score (left) and a learned score (right).
Every null stream draws its own calibration dataset, so each curve estimates $P(\tau_m \leq t)$ averaged over calibration datasets from $2{,}000$ null streams.
The black curve is the oracle detector, which uses the exact pre-change expectation.
}
\label{fig:varym_alarm_time_cdf}
\end{figure}

\newpage
 
\end{document}